\documentclass[nohyperref]{article}

\usepackage{microtype}
\usepackage{graphicx}
\usepackage{subcaption} %
\usepackage{booktabs} %

\usepackage{hyperref}

\usepackage[accepted]{icml2022}

\usepackage{amsmath}
\usepackage{amssymb}
\usepackage{mathtools}
\usepackage{amsthm}

\usepackage[capitalize,noabbrev]{cleveref}

\theoremstyle{plain}
\newtheorem{theorem}{Theorem}[section]

\newtheorem{lemma}[theorem]{Lemma}

\theoremstyle{definition}
\newtheorem{definition}[theorem]{Definition}

\theoremstyle{remark}

\def\E{\mathbf{E}}

\def\H{\mathbf{H}}
\def\I{\mathbf{I}}

\def\M{\mathbf{M}}

\def\P{\mathbf{P}}

\def\R{\mathbf{R}}
\def\S{\mathbf{S}}

\def\U{\mathbf{U}}
\def\V{\mathbf{V}}
\def\W{\mathbf{W}}

\def\Y{\mathbf{Y}}
\def\Z{\mathbf{Z}}

\def\b{\mathbf{b}}

\def\h{\mathbf{h}}

\def\w{\mathbf{w}}
\def\x{\mathbf{x}}
\def\y{\mathbf{y}}

\def\bTheta{\boldsymbol{\Theta}}

\def\bSigma{\boldsymbol{\Sigma}}

\def\btheta{\boldsymbol{\theta}}

\def\bpsi{\boldsymbol{\psi}}

\def\tr{\textrm{Tr}} %

\def\argmin#1{\underset{#1}{\textrm{argmin}}}

\def\minim#1{\underset{#1}{\textrm{min}}}
\def\Exp{\mathbb{E}}
\def\Expp#1{\Exp\left[#1\right]}

\def\0{\mathbf{0}}
\def\1{\mathbf{1}}

\newcommand{\tomt}[1]{\textcolor{black}{#1}}
\newcommand{\tomtm}[1]{\textcolor{black}{#1}}

\newcommand{\tomtr}[1]{\textcolor{black}{#1}}
\newcommand{\tomr}[1]{\textcolor{black}{#1}}

\usepackage[textsize=tiny]{todonotes}

\icmltitlerunning{Extended Unconstrained Features Model for Exploring Deep Neural Collapse}

\begin{document}

\twocolumn[
\icmltitle{Extended Unconstrained Features Model for Exploring Deep Neural Collapse}

\icmlsetsymbol{equal}{*}

\begin{icmlauthorlist}
\icmlauthor{Tom Tirer}{nyu}
\icmlauthor{Joan Bruna}{nyu,nyu2}
\end{icmlauthorlist}

\icmlaffiliation{nyu}{Center for Data Science, New York University,
New York}
\icmlaffiliation{nyu2}{Courant Institute of Mathematical Sciences, New York University, New York}

\icmlcorrespondingauthor{Tom Tirer}{tirer.tom@gmail.com}
\icmlkeywords{Machine Learning, ICML}

\vskip 0.3in
]

\printAffiliationsAndNotice{}  %

\begin{abstract}
The modern strategy for training deep neural networks for classification tasks includes optimizing the network's weights even after the training error vanishes to further push the training loss toward zero. Recently, a phenomenon termed ``neural collapse" (NC) has been empirically observed in this training procedure. Specifically, it has been shown that the learned features (the output of the penultimate layer) of within-class samples converge to their mean, and the means of different classes exhibit a certain tight frame structure, which is also aligned with the last layer's  %
weights. Recent papers have shown that minimizers with this structure emerge when optimizing a simplified ``unconstrained features model" (UFM) with a regularized cross-entropy loss. In this paper, we further analyze and extend the UFM. First, we study the UFM for the regularized MSE loss, and show that the minimizers' features can %
\tomtr{have a more delicate structure} 
than in the cross-entropy case. This affects also the structure of the weights. Then, we extend the UFM by adding another layer of weights as well as ReLU nonlinearity to the model and generalize our previous results. Finally, we empirically demonstrate the usefulness of our nonlinear extended UFM in modeling the %
NC phenomenon that occurs with practical networks.
\end{abstract}

\section{Introduction}
\label{sec:intro}

Deep neural networks (DNNs) have led to a major improvement in classification tasks \citep{krizhevsky2012imagenet,simonyan2014very,he2016deep,huang2017densely}. 
The modern strategy for training %
these networks
includes optimizing the network's weights even after the training error vanishes to further push the training loss toward zero \citep{hoffer2017train,ma2018power,belkin2019does}.

Recently, a phenomenon termed ``neural collapse" (NC) has been empirically observed 
by \citet{papyan2020prevalence}
for such training with cross-entropy loss. Specifically, via experiments on popular network architectures and datasets, %
\citet{papyan2020prevalence} 
showed four components of the NC:
(NC1) The learned features (the output of the penultimate layer) of within-class samples converge to their mean (i.e., the intraclass variance vanishes);
(NC2) After centering by their global mean, the limiting means of different classes exhibit a simplex equiangular tight frame (ETF) structure (see Definition~\ref{def:simplex_etf}); (NC3) The last layer's (classifier) weights are aligned with this
simplex ETF; 
(NC4) As a result, after such a collapse, the classification is based on the nearest class center in feature space.

The empirical work in \citep{papyan2020prevalence} has been followed by papers that theoretically examined the emergence of collapse to simplex ETFs in simplified mathematical frameworks.
Starting from \citep{mixon2020neural},
most of these papers (e.g., \citep{lu2020neural,wojtowytsch2020emergence, fang2021layer,zhu2021geometric}) consider the ``unconstrained features model" (UFM), where the features of the training data after the penultimate layer are treated as free optimization variables (disconnected from the samples).
The rationale behind this model is that modern deep networks are extremely overparameterized and expressive such that their feature mapping can be adapted to any training data (e.g., even to noise \citep{zhang2021understanding}).

While most existing papers consider cross-entropy loss, in this paper we focus on the mean squared error (MSE) loss, which has been recently shown to be powerful also for classification tasks \citep{hui2020evaluation}.
(We note that the occurrence of neural collapse when training practical DNNs with MSE loss, and its positive effects on their performance, have been shown {\em empirically} in a %
very recent paper
\citep{han2021neural}). 
We start with analyzing the (plain) UFM, showing
that 
for the regularized MSE loss the collapsed features can %
\tomtr{have a more delicate structure} 
than in the cross-entropy case (e.g., they may possess also orthogonality), which affects also the structure of the weights.
Then, we extend the UFM by adding another layer of weights as well as ReLU nonlinearity to the model and generalize our previous results. 

\tomtr{Our contributions can be summarized as follows:}
\vspace{-3.5mm}
\begin{itemize}
\itemsep -0.15em 
  \item \tomtr{
  We analyze the minima of the UFM with regularized MSE
loss and show the effect of the bias term on the minimizers’
structured collapse.} 
  \item \tomtr{
  We analyze the minima of a linear extended
UFM and show its limitation in modeling (practical)
depthwise NC behavior.} 
  \item \tomtr{
  We analyze the minima of a
ReLU-based nonlinear extended UFM and show 
the structured collapse of the deepest features.
}   
  \item \tomtr{
We present an asymptotic analysis
of the case where the features are a fixed perturbation
around the structured collapse.} 
  \item \tomtr{
We empirically demonstrate the usefulness of our nonlinear extended UFM in modeling the %
NC phenomenon that occurs in the training of practical networks.}
\end{itemize}

\section{Background and Related Work}
\label{sec:prelim}

In this section, we provide more details on the empirical NC phenomenon and its analysis via the unconstrained features model.

Consider a classification task with $K$ classes and $n$ training samples per class, i.e., overall $N:=Kn$ samples. 
Let us denote by $\y_{k} \in \mathbb{R}^K$ the one-hot vector with 1 in its $k$-th entry %
and by $\x_{k,i} \in \mathbb{R}^D$ the $i$-th training sample of the $k$-th class.
Most DNN-based classifiers can be modeled as 
$$
\bpsi_{\bTheta}(\x) = \W \h_{\btheta}(\x) + \b,
$$
where $\h_{\btheta}(\cdot):\mathbb{R}^D \xrightarrow{} \mathbb{R}^d$ is the feature mapping ($d\geq K$), and $\W = [\w_{1}, \ldots, \w_{K} ]^\top \in \mathbb{R}^{K \times d}$ ($\w_{k}^\top$ denotes the $k$-th row of $\W$) and $\b \in \mathbb{R}^K$ are the last layer’s classifier matrix and bias, respectively.
$\bTheta=\{\W,\b,\btheta\}$ is the set of the trainable network parameters, which includes the parameters $\btheta$ of a nonlinear compositional feature mapping (e.g., %
$\h_{\btheta}(\x)=\sigma(\W_L(\ldots \sigma(\W_2 \sigma(\W_1\x))\ldots)$ where $\sigma(\cdot)$ is an element-wise nonlinear function).

The network parameters are obtained by %
minimizing an empirical risk of the form
\begin{align}
\label{Eq_prob_main_practice}
    \minim{\bTheta} \,\,  &\frac{1}{Kn} \sum_{k=1}^K \sum_{i=1}^n \mathcal{L} \left ( \W \h_{\btheta}(\x_{k,i}) + \b , \y_k \right ) 
    + \mathcal{R} \left ( \bTheta \right ),
\end{align}
where $\mathcal{L}(\cdot,\cdot)$ is a loss function (e.g., cross-entropy or MSE) and $\mathcal{R}(\cdot)$ is a regularization term (e.g., squared $L_2$-norm).
Let us denote the feature vector of the $i$-th training sample of the $k$-th class by $\h_{k,i}$ (i.e., $\h_{k,i}=\h_{\btheta}(\x_{k,i})$),

We now define the notions of (within-class/intraclass) feature collapse and the simplex ETF. We use $\I_K$ to denote the $K \times K$ identity matrix, $\1_K$ to denote the all-ones vector of size $K \times 1$, and $[K]$ to denote the set $\{1,2,...,K\}$.

\begin{definition}[Collapse]
\label{def:collapse}
We
say that the training phase exhibits a (within-class) collapse if all the feature vectors of each class are mapped to a single point, i.e.,
$$
\h_{k,i_1} = \h_{k,i_2}
$$
for all $k \in [K]$ and $i_1,i_2$ training samples of the $k$-th class.
\end{definition}

\begin{definition}[Simplex ETF]
\label{def:simplex_etf}
The standard simplex equiangular tight frame (ETF) is a collection of points in $\mathbb{R}^K$ specified by the columns of
$$
\M = \sqrt{\frac{K}{K-1}} \left ( \I_K - \frac{1}{K} \1_K \1_K^\top  \right ).
$$
Consequently, the standard simplex ETF obeys
$$
\M^\top \M = \M \M^\top  = \frac{K}{K-1} \left ( \I_K - \frac{1}{K} \1_K \1_K^\top  \right ).
$$
In this paper, we consider a (general) simplex ETF as a collection of points in $\mathbb{R}^d$ ($d \geq K$) specified %
by the columns of $\tilde{\M} \propto \sqrt{\frac{K}{K-1}} \P \left ( \I_K - \frac{1}{K} \1_K \1_K^\top  \right )$, where $\P \in \mathbb{R}^{d \times K}$ is an orthonormal matrix. %
Consequently, $\tilde{\M}^\top \tilde{\M} \propto \frac{K}{K-1} \left ( \I_K - \frac{1}{K} \1_K \1_K^\top \right )$.
\end{definition}

\citet{papyan2020prevalence} %
empirically showed that
 training networks after reaching zero training error %
leads to collapse of the features: %
they converge to 
$K$ inter-class means 
that 
form a simplex ETF. Moreover, the last layer's weights $\{ \w_k \}$ are also aligned (i.e., equal up to a scalar factor) to the same simplex ETF, and as a result, the classification turns to be based on the nearest class center in feature space.
This ``neural collapse” (NC) behavior has led to many follow-up papers \citep{mixon2020neural,lu2020neural,wojtowytsch2020emergence,fang2021layer,zhu2021geometric,graf2021dissecting,ergen2021revealing,zarka2020separation,ji2021unconstrained}. Some of them include practical implications of the NC phenomenon, such as %
designing layers (multiplication by tight frames followed by soft-thresholding) 
that concentrate within-class features %
\citep{zarka2020separation} or fixing the last layer's weights to be a simplex ETF \citep{zhu2021geometric}.

To mathematically show the emergence of a collapse to simplex ETF, most follow-up papers  %
have considered a simplified framework --- the “unconstrained features model” (UFM), where the features $\{ \h_{k,i} \}$  are treated as free optimization variables %
\begin{align}
\label{Eq_prob_main_uncons}
    \minim{\W,\b,\{\h_{k,i}\}} \,\,  & \frac{1}{Kn} \sum_{k=1}^K \sum_{i=1}^n \mathcal{L} \left ( \W \h_{k,i} + \b , \y_k \right )  \\ \nonumber
    &\hspace{30mm}
    + \mathcal{R} \left ( \W,\b,\{\h_{k,i}\} \right ).
\end{align}
The rationale for considering this model is that modern overparameterizd deep networks can adapt their feature mapping to almost any training data. 
Specifically, \citep{mixon2020neural} considered the unregularized case (no regularization $\mathcal{R}$) where $\mathcal{L}$ is the MSE loss. It is shown there that a %
simplex EFT is (only) {\em a} global minimizer. However, without penalizing the optimization variables it is easy to see that there are infinitely many global minimizers of different structures (which are not necessarily collapses). In fact, experiments with unregularized MSE loss and %
randomly initialized gradient descent typically %
convergence to non-collapse global minimizers.
(See the dependency on the initialization in the experiments in \citep{mixon2020neural}). 
Other works considered \eqref{Eq_prob_main_uncons} under $L_2$-norm regularized (or constrained) cross-entropy loss with or without the bias term \citep{lu2020neural,%
fang2021layer,zhu2021geometric}.
They showed that, in this case, {\em any} global minimizer has the simplex EFT structure. 

In the following section, we first close the gap for the UFM with regularized MSE loss (this loss has been shown to be powerful for classification tasks \citep{hui2020evaluation}). We show that in this case the collapsed features can be more %
structured than in the cross-entropy case.
Then, we turn to mitigate a limitation of the plain UFM, namely, its inability to capture any behavior that happens across depth as it considers only one level of features. 
To tackle this, we extend the UFM by adding another layer of weights as well as nonlinearity to the model and generalize our previous results. %
\tomtr{We note that there is a concurrent work \cite{zhou2022optimization} that also studies the UFM with regularized MSE loss. Yet, they consider only the plain UFM (with no extensions).}

\section{NC for Unconstrained Features Model with Regularized MSE Loss}
\label{sec:plain_ufm}

In this %
section,
we study the optimization of the %
UFM 
with regularized MSE loss. 
Let $\H = \left [ \h_{1,1}, \ldots, \h_{1,n}, \h_{2,1}, \ldots , \h_{K,n}  \right ] \in \mathbb{R}^{d \times Kn}$ be the (organized) unconstrained features matrix, associated with the one-hot vectors matrix
$\Y  = \I_K \otimes \1_{n}^\top \in \mathbb{R}^{K \times Kn}$, where $\otimes$ denotes the Kronecker product.
We consider the optimization problem 
\begin{align}
\label{Eq_prob_main}
    &\minim{\W,\H, \b} \,\, \frac{1}{2Kn} \sum_{k=1}^{K} \sum_{i=1}^{n} \| \W \h_{k,i} +\b - \y_k \|_2^2 \\ \nonumber 
&\hspace{5mm}    + 
    \frac{\lambda_W}{2} \sum_{k=1}^{K} \|\w_{k}\|_2^2 + \frac{\lambda_H}{2} \sum_{k=1}^{K} \sum_{i=1}^{n} \|\h_{k,i}\|_2^2 + \frac{\lambda_b}{2} \|\b\|_2^2 \nonumber \\  
    &=\minim{\W,\H, \b} \,\,  \frac{1}{2Kn} \| \W \H + \b \1_N^\top - \Y \|_F^2 \\ \nonumber 
&\hspace{5mm}
+ \frac{\lambda_W}{2} \|\W\|_F^2 + \frac{\lambda_H}{2} \|\H\|_F^2 + \frac{\lambda_b}{2} \|\b\|_2^2,
\end{align}
where $\lambda_W$, $\lambda_H$, and $\lambda_b$ are positive regularization hyper-parameters and $\|\cdot\|_F$ denotes the Frobenius norm.

We provide complete characterizations of the minimizers for two settings: (i) the {\em bias-free case}, where $\b=\0$ is fixed (equivalently, $\lambda_b \xrightarrow{} \infty$), and (ii) the {\em unregularized-bias case}, where $\lambda_b=0$ and $\b$ can be optimized.
From these results, %
several 
conclusions are deduced also for the %
case 
where $\lambda_b>0$ and $\b$ is optimizable. %

In the following subsections, we show that while in the {\em unregularized-bias case} the features and weights of any global minimizer are aligned in a simplex ETF structure (similarly to the results obtained for the cross-entropy loss both with and without bias),
in the {\em bias-free case} the features and weights of any global minimizer are aligned in an orthogonal frame (OF) structure.
Since any orthogonal frame can trivially be turned into a simplex ETF by reducing its global mean, in a sense, this collapse is more structured than a simplex ETF collapse.
Giving a precise characterization for the minimizers of the bias-free model is important, %
as later, based on these results, we will study an extension of the bias-free UFM, which has another layer of weights and nonlinearity.

{\bf Remark on the optimization procedure.}
Despite the fact that \eqref{Eq_prob_main} is a non-convex problem (due to the multiplication of $\W$ and $\H$), its global minimizers are easily obtained by simple optimization algorithms, such as plain gradient descent. 
This phenomenon follows from the fact that the optimization landscape of matrix factorization with two factors includes only global minima (no local minima) and strict saddle points (roughly speaking, 
such saddle points can be easily escaped from by gradient-based algorithms)
\citep{kawaguchi2016deep,freeman2017topology}.
\subsection{The Bias-Free Case}

We first consider the optimization problem
\begin{align}
\label{Eq_prob}
    \minim{\W,\H} \,\, %
\frac{1}{2Kn} \| \W \H - \Y \|_F^2 + \frac{\lambda_W}{2} \|\W\|_F^2 + \frac{\lambda_H}{2} \|\H\|_F^2, %
\end{align}
which is a special case of \eqref{Eq_prob_main} with a fixed $\b=\0$ (or equivalently, $\lambda_b \xrightarrow{} \infty$). 

The following theorem characterizes the global solutions of \eqref{Eq_prob}, showing that they necessarily have an orthogonal frame (OF) structure.

\begin{theorem}
\label{thm_nc_of}
Let $d \geq K$ and define $c := K \sqrt{n \lambda_H \lambda_W}$. If $c \leq 1$, 
then any global minimizer $(\W^*,\H^*)$ of \eqref{Eq_prob} satisfies
\begin{align}
\label{Eq_thm_1}
    &\h_{k,1}^* = \ldots = \h_{k,n}^* =: \h_{k}^*, \,\,\,\,\, \forall k \in [K], \\
\label{Eq_thm_2}    
    &\| \h_{1}^* \|_2^2 = \ldots = \| \h_{K}^* \|_2^2 =: \rho = (1-c) \sqrt{\frac{\lambda_W}{n \lambda_H}}, \\ 
\label{Eq_thm_3}    
    &\left [ \h_{1}^*, \ldots, \h_{K}^* \right ]^\top \left [ \h_{1}^*, \ldots, \h_{K}^* \right ] = \rho \I_K, \\ 
\label{Eq_thm_4}    
    &\w_k^* = \sqrt{n \lambda_H / \lambda_W} \h_{k}^*, \,\,\,\,\, \forall k \in [K].
\end{align}
If $c > 1$, then \eqref{Eq_prob} is minimized by $(\W^*,\H^*)=(\0,\0)$. 
\end{theorem}

\begin{proof}
See Appendix~\ref{app:proof1}.
The proof is based on lower bounding the objective by a sequence of inequalities that hold with equality if and only if the stated conditions are satisfied.
\end{proof}

Let us dwell on the implication of this theorem. Denote $\overline{\H}:=\left [ \h_{1}^*, \ldots, \h_{K}^* \right ] \in \mathbb{R}^{d \times K}$. 
In the theorem, \eqref{Eq_thm_1}
implies that the columns of $\H^*$ collapse to the columns of $\overline{\H}$ and 
\eqref{Eq_thm_4} implies
that the rows of $\W^*$ are aligned with the columns of $\overline{\H}$. That is,
\begin{align}
\label{Eq_thm_alt}
    &\H^* = \overline{\H} \otimes \1_n^\top \\ \nonumber
    &\W^* = \sqrt{n \lambda_H / \lambda_W} \overline{\H}^\top.
\end{align}
The consequence of \eqref{Eq_thm_2}, \eqref{Eq_thm_3} and \eqref{Eq_thm_4} is that
\begin{align}
\label{Eq_thm_5}
&\W^*\W^{*\top} = \frac{n \lambda_H}{\lambda_W} \rho \I_K = (1-c) \sqrt{ \frac{n \lambda_H}{\lambda_W} } \I_K, \\ 
\label{Eq_thm_6}
&\W^*\H^* = \sqrt{\frac{n \lambda_H}{\lambda_W}} \rho \I_K \otimes \1_n^\top = (1-c) \I_K \otimes \1_n^\top. 
\end{align}

Note that the collapse of $\W^*$ and $\H^*$ here, in the case of bias-free regularized MSE loss, is to an orthogonal frame (as $\overline{\H}^\top \overline{\H} = \rho \I_K$). %
Yet, by
defining the global feature mean $\h_G^* := \frac{1}{N}\sum_{k=1}^K \sum_{i=1}^n \h_{k,i}^* = \frac{1}{K} \sum_{k=1}^K \h_k^*$, 
trivially, we have that $\overline{\H} - \h_G^* \1_K^\top = \left [ \h_{1}^*-\h_G^*, \ldots, \h_{K}^*-\h_G^* \right ]$ is a simplex ETF. This follows from
\begin{align}
&\left ( \overline{\H} - \h_G^* \1_K^\top \right )^\top \left ( \overline{\H} - \h_G^* \1_K^\top \right ) \\ \nonumber
&= \left ( \I_K - \frac{1}{K} \1_K \1_K^\top \right )^\top \overline{\H}^\top \overline{\H}  \left ( \I_K - \frac{1}{K} \1_K \1_K^\top \right ) \\ \nonumber
&= \rho \left ( \I_K - \frac{1}{K} \1_K \1_K^\top \right )^\top \left ( \I_K - \frac{1}{K} \1_K \1_K^\top \right ) \\ \nonumber
&= \rho \left ( \I_K - \frac{1}{K} \1_K \1_K^\top \right ),
\end{align}
where we used $\h_G^* = \frac{1}{K}\overline{\H}\1_K$.
In that sense, %
$\H^*$ here is more structured 
than in the 
results reported by
previous works that considered the UFM with regularized/constrained cross-entropy loss  \citep{lu2020neural,fang2021layer,zhu2021geometric}, 
where the collapse of $\W^*$ and $\H^*$ is to a simplex ETF. 
\subsection{The Unregularized-Bias Case}

We next turn to consider the optimization problem
\begin{align}
\label{Eq_prob_bias}
    &\minim{\W,\H, \b} \,\, %
\frac{1}{2Kn} \| \W \H + \b \1_N^\top - \Y \|_F^2 \\ \nonumber
& \hspace{10mm}
+ \frac{\lambda_W}{2} \|\W\|_F^2 + \frac{\lambda_H}{2} \|\H\|_F^2, %
\end{align}
which is a special case of \eqref{Eq_prob_main} when %
$\lambda_b=0$. 

The following theorem characterizes the global solutions of \eqref{Eq_prob_bias}, showing that they necessarily have a simplex ETF structure.

\begin{theorem}
\label{thm_nc_simplex_etf}
Let $d \geq K$ and define $c := K \sqrt{n \lambda_H \lambda_W}$. If $c \leq 1$, 
then any global minimizer $(\W^*,\H^*,\b^*)$ of \eqref{Eq_prob_bias} satisfies %
\begin{align}
\label{Eq_thm2_0}
    &\b^*=\frac{1}{K}\1_K, \\
\label{Eq_thm2_1}
    &\h_{k,1}^* = \ldots = \h_{k,n}^* =: \h_{k}^*, \,\,\,\,\, \forall k \in [K], \\
\label{Eq_thm2_1g}
    &\h_G^* := \frac{1}{N}\sum_{k=1}^K \sum_{i=1}^n \h_{k,i}^* = \frac{1}{K} \sum_{k=1}^K \h_k^* = \0, \\
\label{Eq_thm2_2}    
    &\| \h_{1}^* \|_2^2 = \ldots = \| \h_{K}^* \|_2^2 =: \rho = \frac{(1-c)(K-1)}{K} \sqrt{\frac{\lambda_W}{n \lambda_H}}, \\ 
\label{Eq_thm2_3}    
    &\left [ \h_{1}^*, \ldots, \h_{K}^* \right ]^\top \left [ \h_{1}^*, \ldots, \h_{K}^* \right ] = \rho \frac{K}{K-1} \left ( \I_K - \frac{1}{K} \1_K \1_K^\top \right ), \\ 
\label{Eq_thm2_4}    
    &\w_k^* = \sqrt{n \lambda_H / \lambda_W} \h_{k}^*, \,\,\,\,\, \forall k \in [K].
\end{align}
If $c > 1$, then \eqref{Eq_prob_bias} is minimized by $(\W^*,\H^*,\b^*)=(\0,\0,\frac{1}{K}\1_K)$. 
\end{theorem}

\begin{proof}
See Appendix~\ref{app:proof2}.
Similarly to the previous theorem, the proof is based on lower bounding the objective by a sequence of inequalities that hold with equality if and only if the stated conditions are satisfied.
\end{proof}

The consequence of \eqref{Eq_thm2_2}, \eqref{Eq_thm2_3} and \eqref{Eq_thm2_4} is that
\begin{align}
\label{Eq_thm2_5}
&\W^*\W^{*\top} = \frac{n \lambda_H}{\lambda_W} \rho \frac{K}{K-1} \left ( \I_K - \frac{1}{K} \1_K \1_K^\top \right ), \\ 
\label{Eq_thm2_6}
&\W^*\H^* = \sqrt{\frac{n \lambda_H}{\lambda_W}} \rho \frac{K}{K-1} \left ( \I_K - \frac{1}{K} \1_K \1_K^\top \right ) \otimes \1_n^\top. 
\end{align}
Note 
that the results in Theorem~\ref{thm_nc_simplex_etf} (contrary to those in Theorem~\ref{thm_nc_of}) resemble the results that have been obtained for the cross-entropy loss (both with and without bias). However, as far as we know, no such theorem has been reported for the case of MSE loss.

{\bf Remark on the regularized-bias case.} From Theorems~\ref{thm_nc_of} and \ref{thm_nc_simplex_etf}, we get the following facts about the global minimizers.
In the bias-free case ($\lambda_b \xrightarrow{} \infty$), $\H^*$ has an OF structure, and trivially, if we subtract from it the global feature mean $\h_G^*$, we get that $\H^*-\h_G^* \1_K$ has a simplex ETF structure.
In the unregularized-bias case ($\lambda_b = 0$), $\H^*$ has a simplex ETF structure. Trivially, this is also the structure of $\H^*-\h_G^* \1_K$, as the global feature mean $\h_G^*$ equals zero in this case. 
In both cases, $\W^{*\top}$ is aligned with $\H^*$, i.e., it is an OF in the bias-free case and a simplex ETF in the unregularized-bias case.
The consequence of these results\footnotemark 
is that for the fully regularized MSE loss, where $0<\lambda_b< \infty$, the global minimizers may have $\H^*$ and $\W^{*\top}$ that are neither a simplex EFT nor an OF. 
Yet, we empirically observed that still $\W^{*\top}$ is aligned with $\H^*$ and that $\H^*-\h_G^* \1_K$ is a simplex ETF (as may be expected, because these two properties hold in both extreme settings of $\lambda_b$).

\footnotetext{In the UFM, note that the (within-class) collapse of the global minimizers (i.e., $\h_{i,k}^*=\h_k^*$ for all $i\in[n]$) is a consequence of 
the symmetry of the loss and the regularization terms w.r.t.~the sample index, which, in our proofs, is exploited by 
attaining Jensen's inequality %
when averaging 
over $i\in[n]$. %
Thus, it does not depend on whether we regularize the bias term.} 

\section{Extended Unconstrained Features Model}
\label{sec:ext_ufm}

The UFM, which considers only one level of features, cannot capture any behavior that happens across depth. 
Therefore, in this section, we extend %
this model, 
first with another layer of weights, and then with the nonlinear ReLU activation between the two layers of weights.

\subsection{Unconstrained Features Model With an Additional Layer}

Consider the following optimization problem that corresponds to %
an extended UFM with two layers of weights,
\begin{align}
\label{Eq_prob_deeper}
    &\minim{\W_2,\W_1,\H_1} \,\,  \frac{1}{2Kn} \| \W_2 \W_1 \H_1  - \Y \|_F^2 \\ \nonumber 
& \hspace{10mm}
+ \frac{\lambda_{W_2}}{2} \|\W_2\|_F^2 + \frac{\lambda_{W_1}}{2} \|\W_1\|_F^2 + \frac{\lambda_{H_1}}{2} \|\H_1\|_F^2,
\end{align}
where $\lambda_{W_2}$, $\lambda_{W_1}$, and $\lambda_{H_1}$ are regularization hyper-parameters, and
$\W_2 \in \mathbb{R}^{K \times d}$, $\W_1 \in \mathbb{R}^{d \times d}$, $\H_1 \in \mathbb{R}^{d \times N}$. 
Observe the similarity between \eqref{Eq_prob_deeper} and \eqref{Eq_prob}, where $(\W,\H)$ in \eqref{Eq_prob} are replaced by $(\W_2,\W_1\H_1)$ or by $(\W_2\W_1,\H_1)$.
Yet, the similarity is only {\em partial} because, e.g., if we plug $(\W,\H)=(\W_2,\W_1\H_1)$ in \eqref{Eq_prob} we get a regularization term $\| \W_1\H_1 \|_F^2$ rather than separated $\| \W_1 \|_F^2$ and $\| \H_1 \|_F^2$.
To the best of our knowledge, characterization of the minimizers of a multilayer extension of the unconstrained features model
has not been done so far. 

{\bf Remark on the optimization procedure.}
While both \eqref{Eq_prob_deeper} and \eqref{Eq_prob} are non-convex problems, obtaining the global minimizers of \eqref{Eq_prob_deeper} is more challenging in practice
(e.g., requires 
careful initializations). 
This follows from the fact that the optimization landscapes of matrix factorization with three of more factors (or equivalently, non-shallow linear neural networks) include also non-strict saddle points, which entangle gradient-based methods \citep{kawaguchi2016deep}.

The following theorem characterizes the global solutions of \eqref{Eq_prob_deeper}. %
It shows that the orthogonal frame structure of the solutions is maintained despite the %
intermediate 
weight matrix that has been added.
Here ``$\propto$" denotes proportional, i.e., equal up to a positive scalar factor.

\begin{theorem}
\label{thm_nc_of_deeper}
Let $d > K$ and $(\W_2^*,\W_1^*,\H_1^*)$ be a global minimizer of \eqref{Eq_prob_deeper}.
Then, both $\H_1^*$ and $\W_1^*\H_1^*$ collapse to orthogonal $d \times K$ frames. Also, both $(\W_2^*\W_1^*)^\top$ and $\W_2^{*\top}$ are orthogonal %
$d \times K$ 
matrices, where $(\W_2^*\W_1^*)^\top$ is aligned with $\H_1^{*}$ and $\W_2^{*\top}$ is aligned with $\W_1^*\H_1^*$.
Formally, we have that
$\H_1^* = \overline{\H}_1 \otimes \1_n^\top$ for some $\overline{\H}_1 \in \mathbb{R}^{d \times K}$, and
$$
(\W_2^*\W_1^*) \overline{\H}_1 \propto \overline{\H}_1^\top \overline{\H}_1 \propto  (\W_2^*\W_1^*) (\W_2^*\W_1^*)^\top \propto  \I_K.
$$
Similarly, we have that
$\H_2^*:=\W_1^*\H_1^* = \overline{\H}_2 \otimes \1_n^\top$ for some $\overline{\H}_2 \in \mathbb{R}^{d \times K}$, and
$$
\W_2^* \overline{\H}_2 \propto \overline{\H}_2^\top \overline{\H}_2 \propto  \W_2^* \W_2^{*\top} \propto  \I_K.
$$
\end{theorem}

\begin{proof}
See Appendix~\ref{app:proof3}.
The proof is based on connecting the minimization of the three-factors objective with two sub-problems that include two-factors objectives. More specifically, the sum of the Frobenius norm regularization of two matrices is lower bounded (with attainable equality) by a suitably scaled nuclear norm of their multiplication, and the minimizers of the latter formulation, which can be expressed by the minimizers of the original problem, are analyzed.
\end{proof}

{\bf Remark on the choice of loss function.}
The proof of Theorem~\ref{thm_nc_of_deeper} mostly depends on handling the regularization terms when transforming the problem into two sub-problems, and can be potentially modified to the case where the cross-entropy loss is used instead of MSE.
Thus, a similar theorem can be stated for cross-entropy loss, for which it is known that the minimizers of the plain UFM collapse as well \citep{zhu2021geometric}.
Naturally, in such a statement the collapse will be to a simplex ETF rather than to an OF.
Indeed, we empirically observed that also when using the cross-entropy loss in \eqref{Eq_prob_deeper}, the global minimizers $\W_1^*\H_1^*$ and $\H_1^*$ collapse to a simplex ETF structure.

{\bf Discussion.}
\tomt{In practical ``well-trained" DNNs (e.g., see Figure~\ref{fig:mnist_mse_noBias_and_CE}  in the experiments section): (1) structured collapse appears only in the deepest features; (2) %
decrease in within-class variability 
is obtained monotonically along the depth of the network\footnote{\tomtr{Similar depthwise progressive variability decrease is empirically observed also in \cite{papyan2020traces}}.}}. 
However, %
Theorem~\ref{thm_nc_of_deeper} shows
the emergence of structured (orthogonal) collapse {\em simultaneously} at the two levels of unconstrained features of the model in \eqref{Eq_prob_deeper} --- both at the deeper $\H_2:=\W_1\H_1$ and at the shallower $\H_1$ --- which does not fit (1).
Moreover, the linear link between $\H_2$ and $\H_1$ implies that they have %
\tomr{similar} 
within-class variability measured by the metric $NC_1$ (defined in \eqref{Eq_NC1} below) %
\tomr{when} 
the columns of $\H_1$ are not in the null space of $\W_1$. 
This hints that $\H_1$ and $\H_2$ may have similar values/slopes for their $NC_1$ metric %
after random initialization and along gradient-based optimization 
(see Appendix~\ref{app:lemma_nc1} for more details). 
Yet, this does not fit (2).
Therefore, extending the model to two levels of features without the addition of a non-linearity still cannot capture the 
behavior of practical DNNs across layers.
This encourages us to further extend the model by adding a nonexpansive nonlinear activation function (ReLU) %
between $\W_2$ and $\W_1$, that naturally breaks the similarity between the two levels of features. 

\subsection{Non-Linear Unconstrained Features Model}

In this section, we turn to consider a nonlinear version of the unconstrained features model that has been stated in \eqref{Eq_prob_deeper}. Specifically, using the same notation as \eqref{Eq_prob_deeper}, we consider the optimization problem
\begin{align}
\label{Eq_prob_nonlin}
    &\minim{\W_2,\W_1,\H_1} \,\,  \frac{1}{2Kn} \| \W_2 \sigma( \W_1 \H_1 ) - \Y \|_F^2 \\ \nonumber 
& \hspace{10mm}
+ \frac{\lambda_{W_2}}{2} \|\W_2\|_F^2 + \frac{\lambda_{W_1}}{2} \|\W_1\|_F^2 + \frac{\lambda_{H_1}}{2} \|\H_1\|_F^2,
\end{align}
where $\sigma(\cdot)=\mathrm{max}(0,\cdot)$ is the element-wise ReLU function.

The following theorem characterizes the global solutions of \eqref{Eq_prob_nonlin} by %
exploiting the similarity of this 
problem to the one in \eqref{Eq_prob_deeper}.
It shows that the orthogonal frame structure created by the optimal solution %
$(\W^*,\H^*)=(\W_2^*,\sigma(\W_1^*\H_1^*))$ is maintained despite the nonlinearity that has been added.

\begin{theorem}
\label{thm_nc_of_deeper_nonlin}
Let $d > K$ and $(\W_2^*,\W_1^*,\H_1^*)$ be a global minimizer of \eqref{Eq_prob_nonlin}. 
\tomr{If the nuclear norm equality $\|\W_1^*\H_1^*\|_* = \|\sigma(\W_1^*\H_1^*)\|_*$ holds\footnote{\tomtr{This assumption
is required in our current proof. We verified that it holds in any numerical experiment that we performed.
}}%
, then}
$\sigma(\W_1^*\H_1^*)$ collapses to an orthogonal $d \times K$ frame and $\W_2^{*\top}$ is an orthogonal %
$d \times K$
matrix that is aligned with $\sigma(\W_1^*\H_1^*)$, i.e., 
$\H_2^*:=\sigma(\W_1^*\H_1^*) = \overline{\H}_2 \otimes \1_n^\top$ for some non-negative $\overline{\H}_2 \in \mathbb{R}^{d \times K}$, and
$$
\W_2^* \overline{\H}_2 \propto \overline{\H}_2^\top \overline{\H}_2 \propto  \W_2^* \W_2^{*\top} \propto  \I_K.
$$
\end{theorem}

\begin{proof}
See Appendix~\ref{app:proof4}. 
\tomtr{The proof utilizes the one of Theorem~\ref{thm_nc_of_deeper}. It is based on showing equivalence between two related sub-problems w/ and w/o the ReLU operation.} 
\end{proof}

Note that
the structure of $(\W^*,\H^*)=(\W_2^*,\sigma(\W_1^*\H_1^*))$ is the same as for the model in \eqref{Eq_prob_deeper}, where the non-linearity is absent (yet, here $\H^*$ is obviously also non-negative).
This  %
 analysis 
benefits from the fact that the features are unconstrained, and is in contrast with the usual case, where 
the results obtained for linear models do not carry ``as is" to their non-linear counterparts. 
In 
Section~\ref{sec:exps}
we show that the nonlinearity is necessary for %
capturing the different 
behavior of features in different depths during the collapse of practical networks.

\begin{figure*}[t]
  \centering
  \begin{subfigure}[b]{0.25\linewidth}
    \centering\includegraphics[width=100pt]{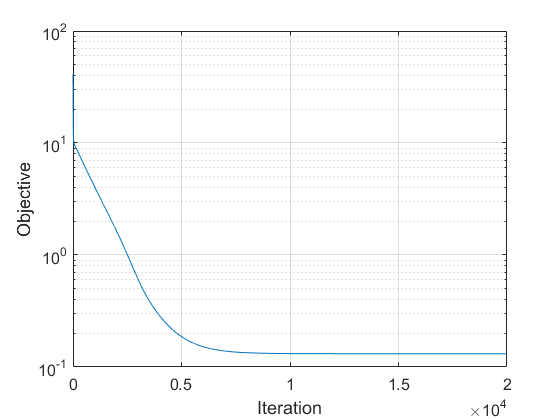}
   \end{subfigure}%
  \begin{subfigure}[b]{0.25\linewidth}
    \centering\includegraphics[width=100pt]{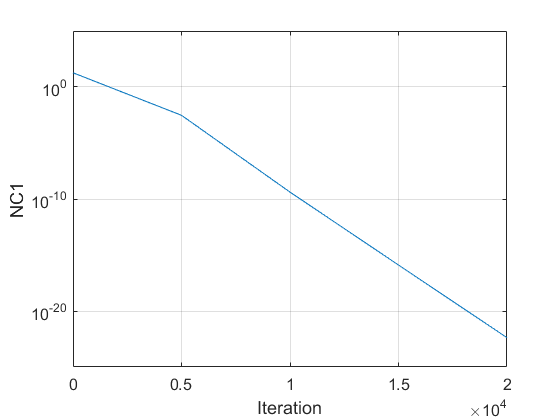}
   \end{subfigure}%
  \begin{subfigure}[b]{0.25\linewidth}
    \centering\includegraphics[width=100pt]{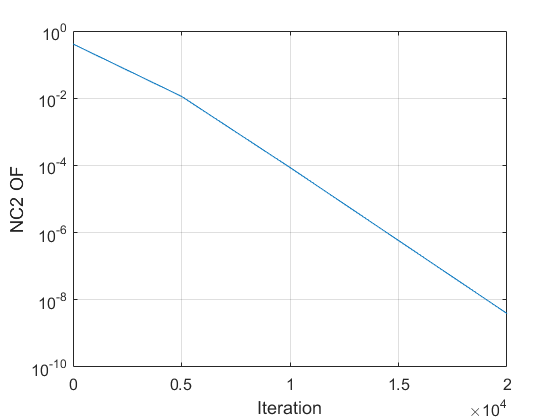}
   \end{subfigure}%
  \begin{subfigure}[b]{0.25\linewidth}
    \centering\includegraphics[width=100pt]{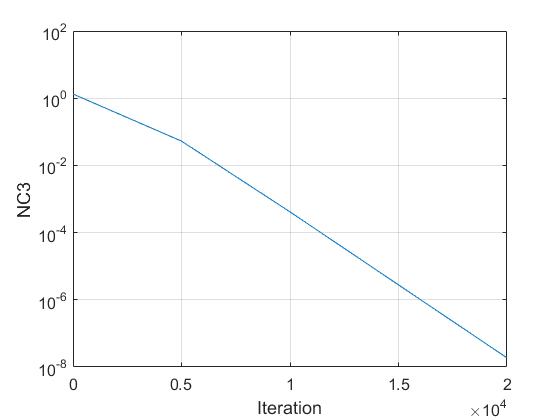}
   \end{subfigure}%
    \caption{
    Verification of Theorem~\ref{thm_nc_of} (MSE loss with no bias).
    From left to right: the objective value, NC1 (within-class variability), NC2 (similarity of the features to OF), and NC3 (alignment between the weights and the features).
    }
\label{fig:thm1}     

\vspace{4mm}

  \centering
  \begin{subfigure}[b]{0.25\linewidth}
    \centering\includegraphics[width=100pt]{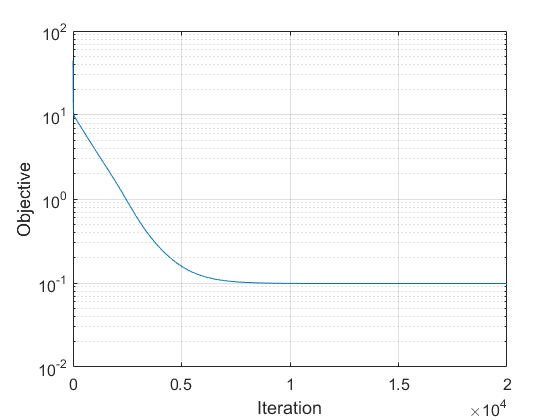}
   \end{subfigure}%
  \begin{subfigure}[b]{0.25\linewidth}
    \centering\includegraphics[width=100pt]{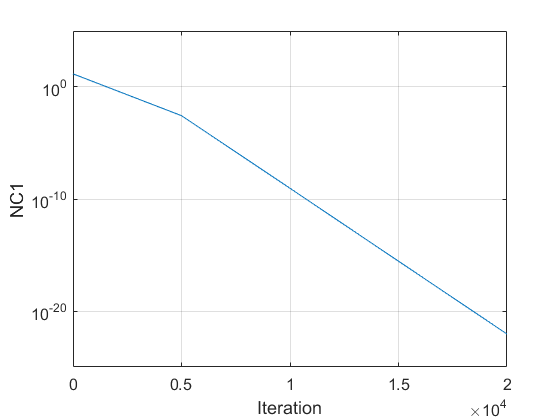}
   \end{subfigure}%
  \begin{subfigure}[b]{0.25\linewidth}
    \centering\includegraphics[width=100pt]{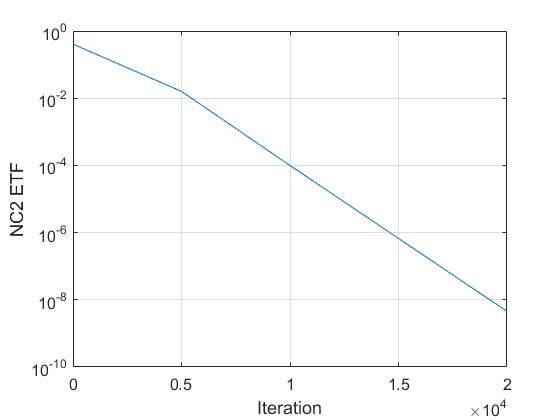}
   \end{subfigure}%
  \begin{subfigure}[b]{0.25\linewidth}
    \centering\includegraphics[width=100pt]{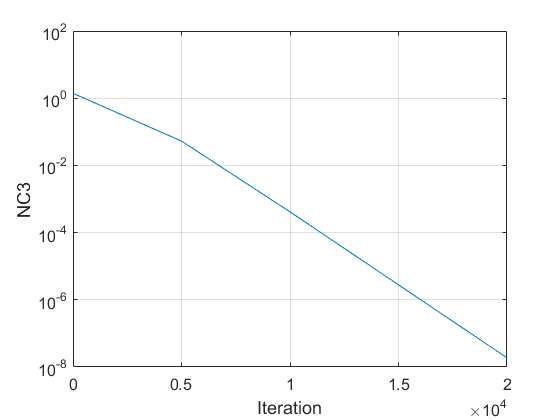}
   \end{subfigure}%
    \caption{
    Verification of Theorem~\ref{thm_nc_simplex_etf} (MSE loss with unregularized bias).
    From left to right: the objective value, NC1 (within-class variability), NC2 (similarity of the features to simplex ETF), and NC3 (alignment between the weights and the features).    
    }
\label{fig:thm2}     
\end{figure*}

\begin{figure*}[t]
  \centering
  \begin{subfigure}[b]{0.25\linewidth}
    \centering\includegraphics[width=100pt]{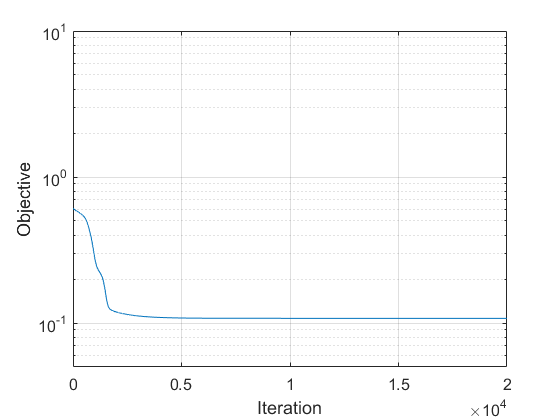} %
   \end{subfigure}%
  \begin{subfigure}[b]{0.25\linewidth}
    \centering\includegraphics[width=100pt]{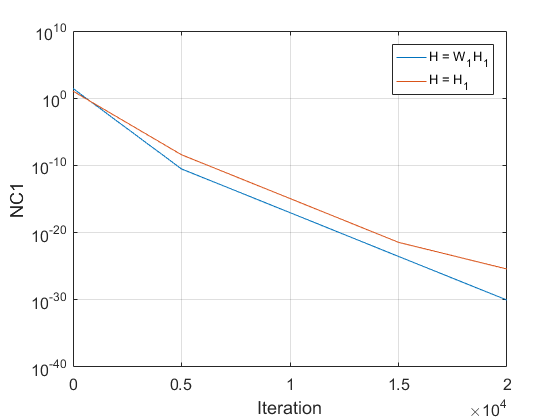}
   \end{subfigure}%
  \begin{subfigure}[b]{0.25\linewidth}
    \centering\includegraphics[width=100pt]{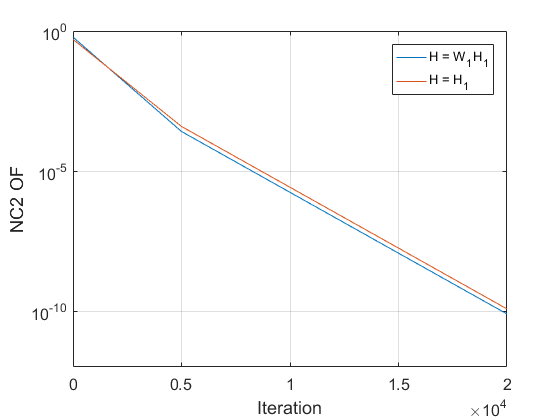}
   \end{subfigure}%
  \begin{subfigure}[b]{0.25\linewidth}
    \centering\includegraphics[width=100pt]{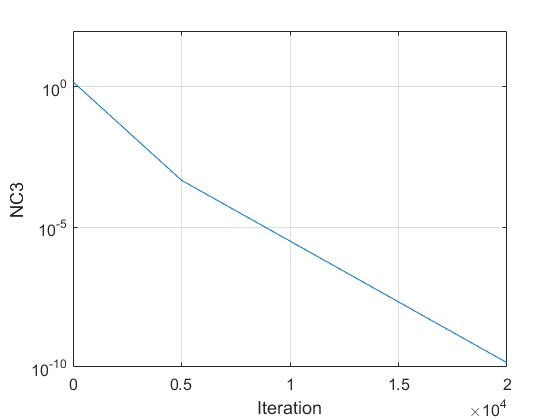}
   \end{subfigure}%
    \caption{
    Verification of Theorem~\ref{thm_nc_of_deeper} %
    (two levels of features). 
    From left to right: the objective value, NC1 (within-class variability), NC2 (similarity of the features to OF), and NC3 (alignment between the weights and the features).
    }
\label{fig:thm3}     

\vspace{4mm}

  \centering
  \begin{subfigure}[b]{0.25\linewidth}
    \centering\includegraphics[width=100pt]{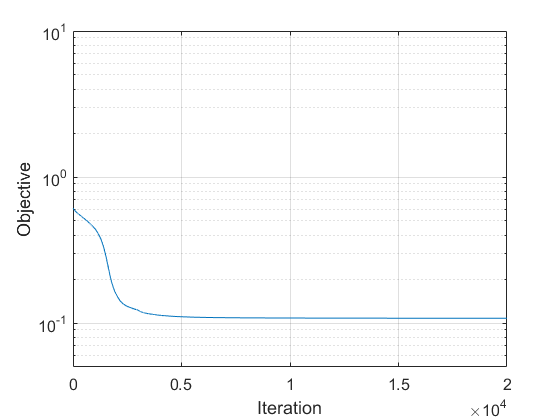}%
   \end{subfigure}%
  \begin{subfigure}[b]{0.25\linewidth}
    \centering\includegraphics[width=100pt]{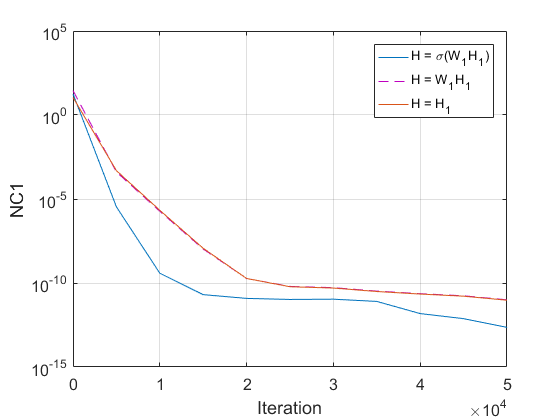}
   \end{subfigure}%
  \begin{subfigure}[b]{0.25\linewidth}
    \centering\includegraphics[width=100pt]{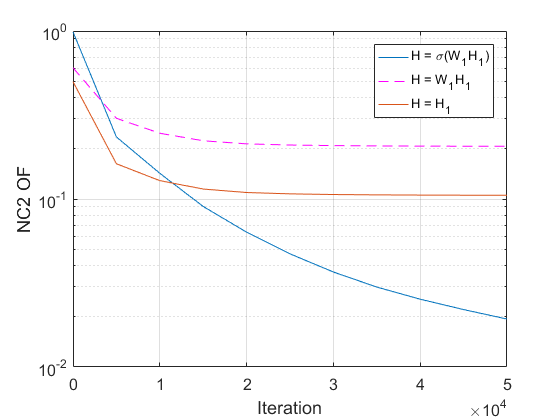}
   \end{subfigure}%
  \begin{subfigure}[b]{0.25\linewidth}
    \centering\includegraphics[width=100pt]{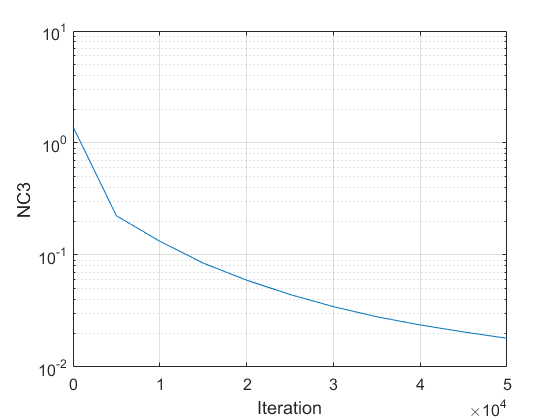}
   \end{subfigure}%
    \caption{
    Verification of Theorem~\ref{thm_nc_of_deeper_nonlin} (two levels of features with ReLU activation).
    From left to right: the objective value, NC1 (within-class variability), NC2 (similarity of the features to OF), and NC3 (alignment between the weights and the features).
    }
\label{fig:thm4}

\vspace{4mm}

  \centering
  \begin{subfigure}[b]{0.25\linewidth}
    \centering\includegraphics[width=107pt]{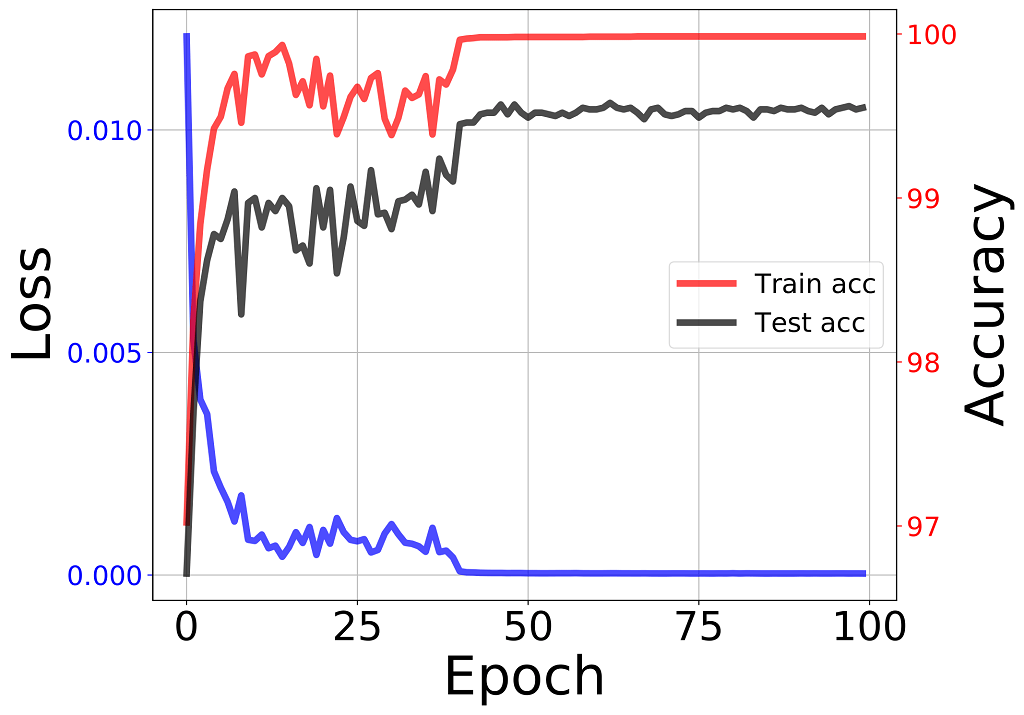}
   \end{subfigure}%
  \begin{subfigure}[b]{0.25\linewidth}
    \centering\includegraphics[width=96pt]{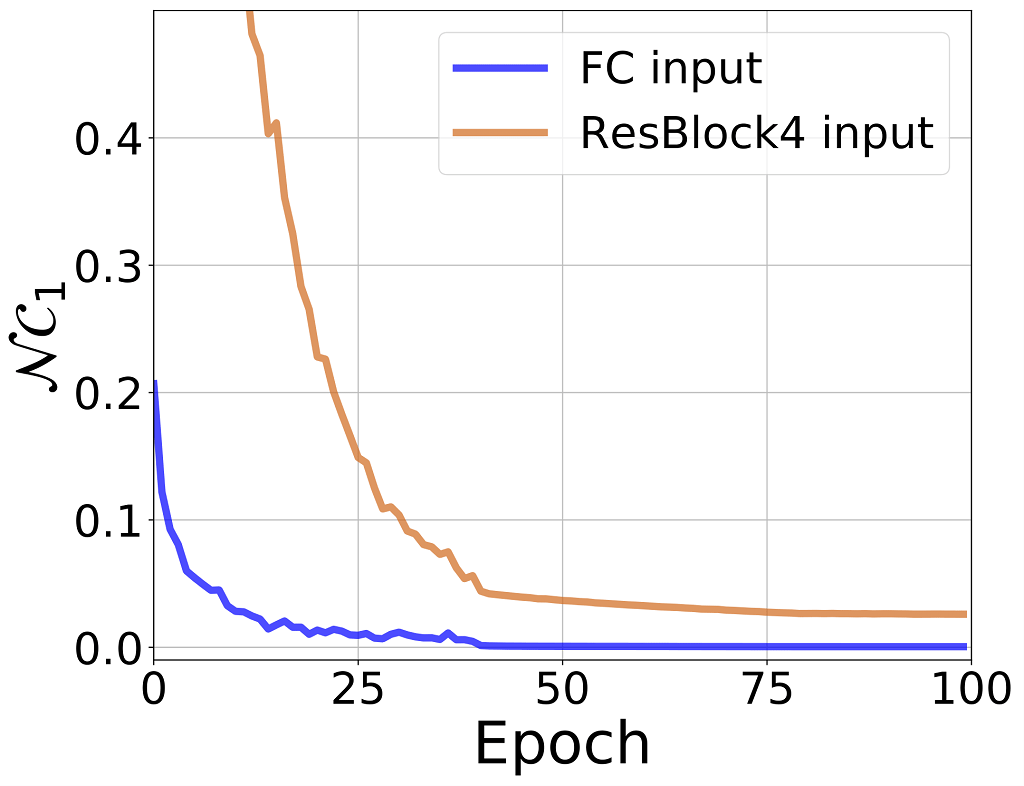}
   \end{subfigure}%
  \begin{subfigure}[b]{0.25\linewidth}
    \centering\includegraphics[width=96pt]{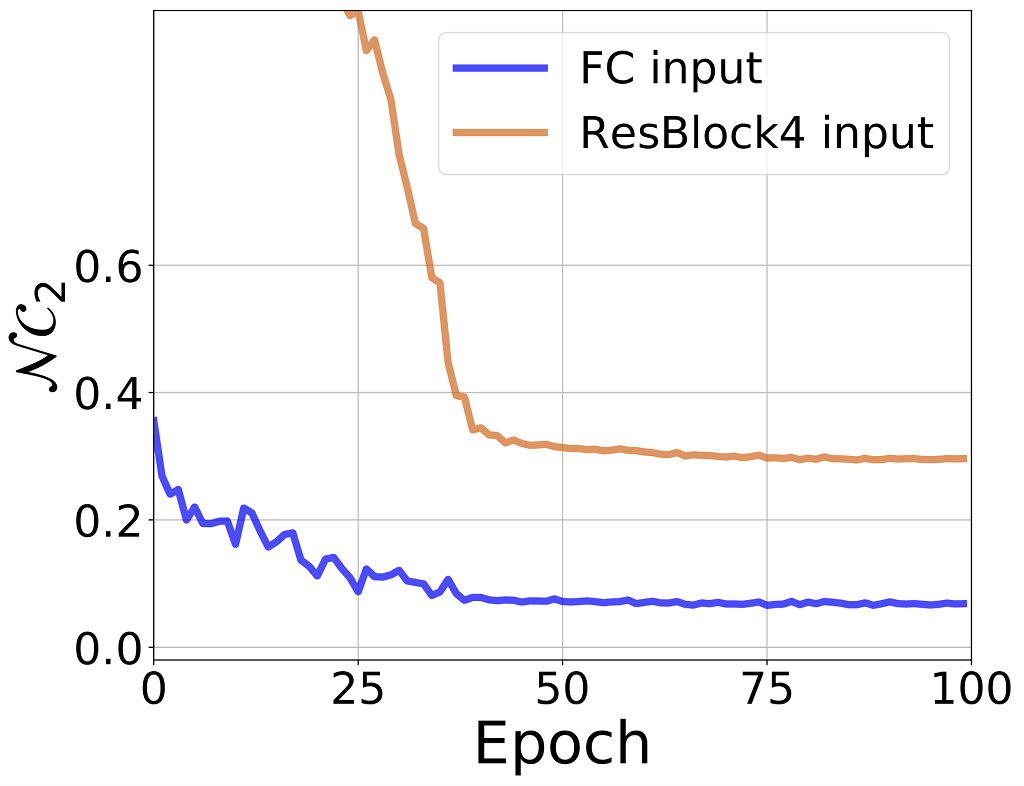}
   \end{subfigure}%
  \begin{subfigure}[b]{0.25\linewidth}
    \centering\includegraphics[width=96pt]{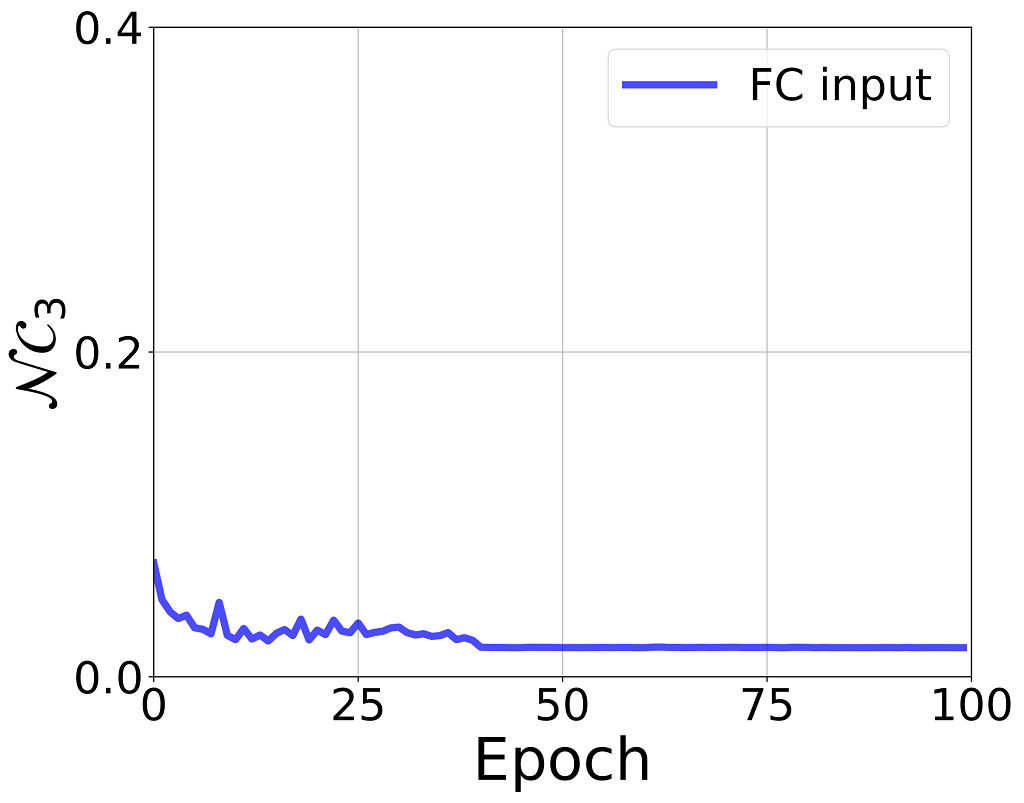}
   \end{subfigure}%
\\
  \begin{subfigure}[b]{0.25\linewidth}
    \centering\includegraphics[width=107pt]{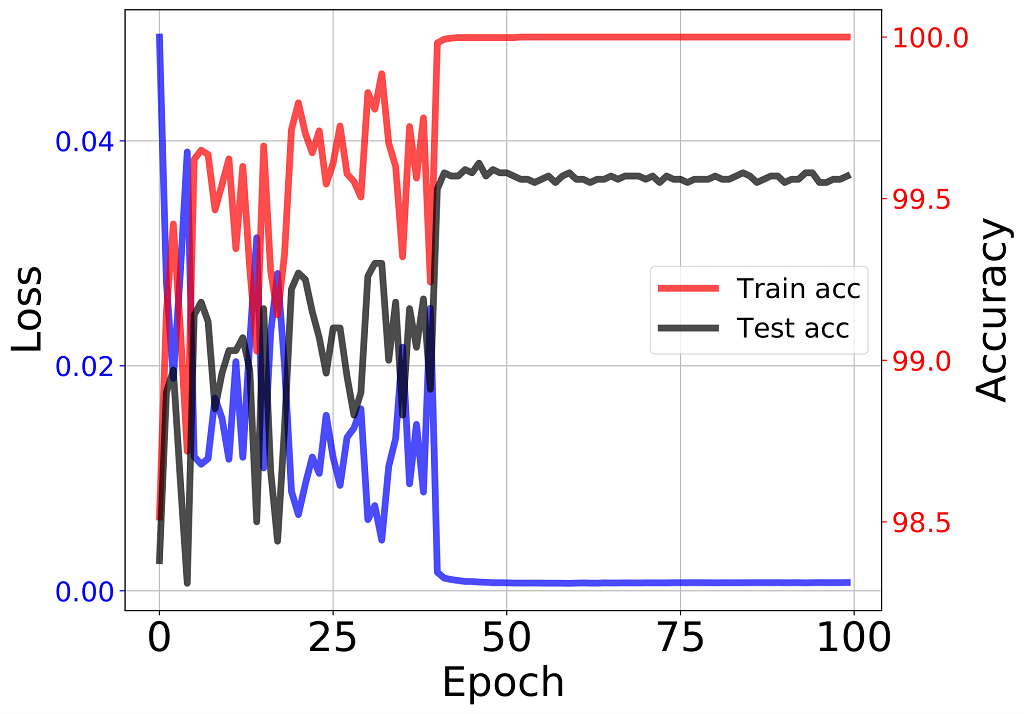}
   \end{subfigure}%
  \begin{subfigure}[b]{0.25\linewidth}
    \centering\includegraphics[width=96pt]{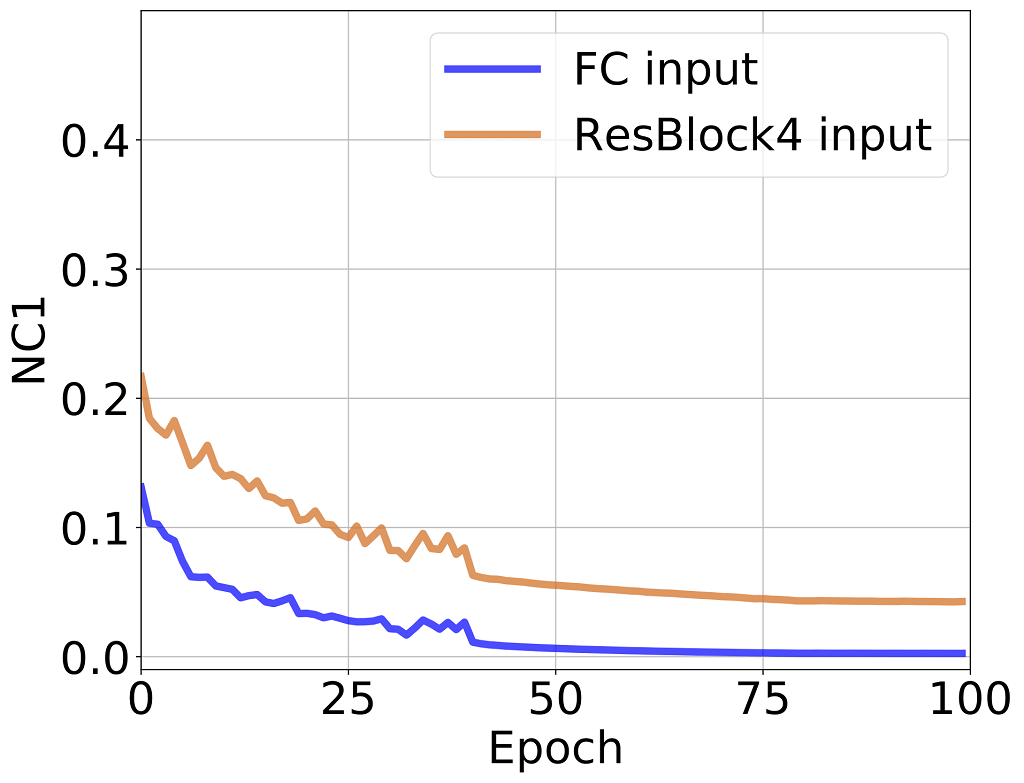}
   \end{subfigure}%
  \begin{subfigure}[b]{0.25\linewidth}
    \centering\includegraphics[width=96pt]{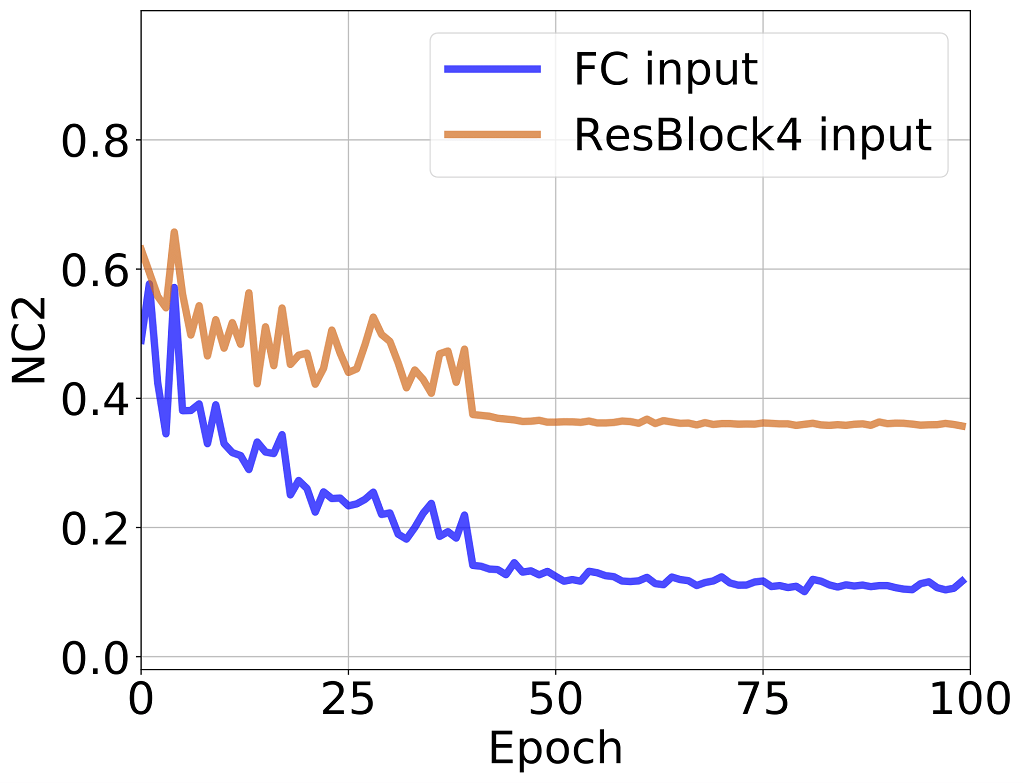}
   \end{subfigure}%
  \begin{subfigure}[b]{0.25\linewidth}
    \centering\includegraphics[width=96pt]{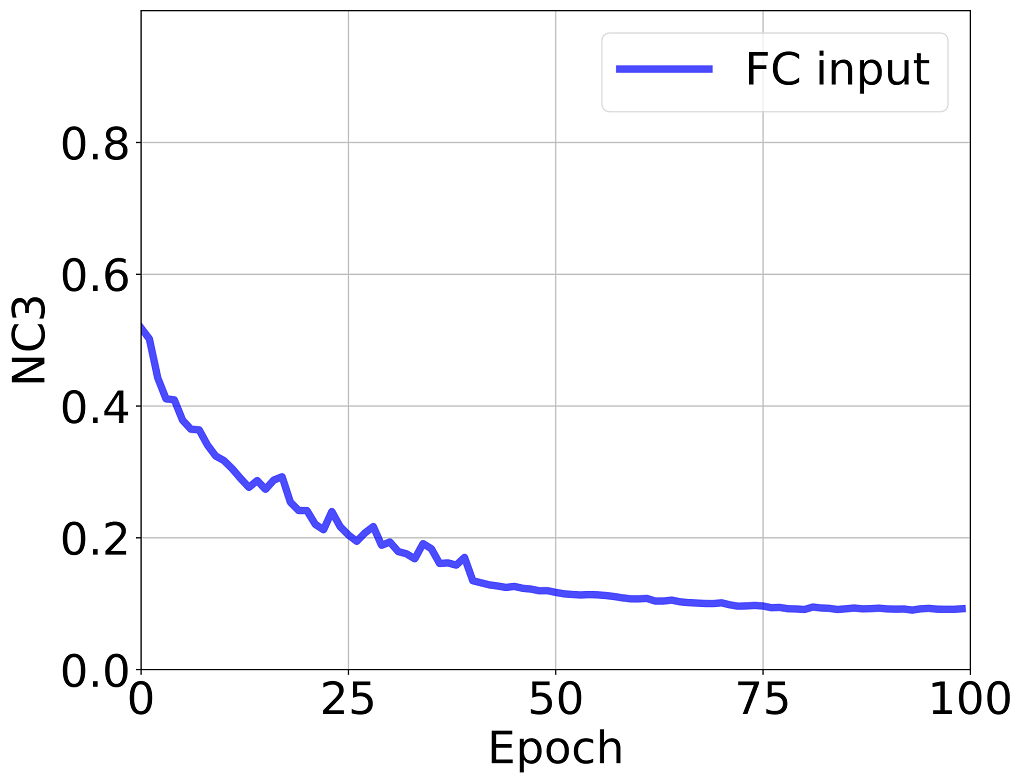}
   \end{subfigure}%
    \caption{
    NC metrics for ResNet18 trained on MNIST. Top: MSE loss, weight decay, and no bias; Bottom: Cross-entropy loss and weight decay.
    From left to right: training's objective value and accuracy, NC1 (within-class variability), NC2 (similarity of the centered features to simplex ETF), and NC3 (alignment between the weights and the features).
    }
\label{fig:mnist_mse_noBias_and_CE}     

\end{figure*}

\section{Toward Generalizing the UFMs Results to %
\tomtr{Other Models}}

\tomtm{
Similar to the existing theoretical works that demonstrate the emergence of collapsed minimizers, in this paper we considered models where the features matrix $\H$ (or $\H_1$) is a free optimization variable.
It is of high interest to make a step forward and instead of freely optimize the features connect  them to some data distribution.} 

While we defer a comprehensive study that links the models to data for future research, in this short section we demonstrate the feasibility of this goal, 
even for the plain UFM, through the following theorem.

\begin{theorem}
\label{thm_nc_of_asymp}
\tomtm{
Consider \eqref{Eq_prob} with $\lambda_H = \frac{\tilde{\lambda}_H}{n}$. Denote by $(\W^*,\H^*)$ a global minimizer of \eqref{Eq_prob} for some $n$. Following Theorem~\ref{thm_nc_of}, observe that $\H^*=\overline{\H}\otimes \1_n^\top$ %
for some $\overline{\H}\in \mathbb{R}^{d \times K}$. %
Let $\tilde{\H}_n := \overline{\H}\otimes \1_n^\top + \E_n$
where $\E_n \in \mathbb{R}^{d \times Kn}$ whose entries are i.i.d.~random variables with zero mean, variance $\sigma_e^2$, \tomt{and finite fourth moment}. 
Let
\begin{equation}
\label{Eq_fixed_H}
\hat{\W}_n = \argmin{\W} \,\, %
\frac{1}{2Kn} \| \W \tilde{\H}_n - \Y \|_F^2 + \frac{\lambda_W}{2} \|\W\|_F^2.
\end{equation}
We have that $\hat{\W}_n \xrightarrow[n \xrightarrow{} \infty]{a.s.} \frac{1}{1+\sigma_e^2 K \sqrt{\tilde{\lambda}_H/\lambda_W}} \W^*$.}  %
\end{theorem}

\begin{proof}
\tomtm{
See Appendix~\ref{app:proof_asymp}. 
The proof exploits the fact that $\hat{\W}_n$ has a closed-form expression (a function of the features matrix) that allows linking it to $\W^*$.}  
\end{proof}

Theorem~\ref{thm_nc_of_asymp} shows that as the
number of samples tend to infinity we have that properties of the optimal weights such as the orthogonal structure and the alignment with $\overline{\H}^\top$ (stated for $\W^*$ in Theorem~\ref{thm_nc_of}) are restored even with a fixed non-collapsed features matrix. %

\tomtm{
As discussed in Appendix~\ref{app:proof_asymp_intuitive}, the intuition that the asymptotic consequence of the %
deviation from ``perfectly" collapsed features
will only be some attenuation of $\W^*$ can also be seen from expending the quadratic term in \eqref{Eq_fixed_H} and eliminating the terms that are linear in the zero-mean $\E_n$. This intuition applies also for the extended UFMs  %
with fixed features (where no closed-form minimizers exist).} 
\section{Numerical Results}
\label{sec:exps}

In this section, %
we corroborate our theoretical results with experiments.
For each setting that is considered in %
the theorems of Sections~\ref{sec:plain_ufm} and \ref{sec:ext_ufm} 
we tune a gradient descent scheme to reach a global minimizer. We plot the optimization's objective value curve at different iterations, as well as 
several metrics that measure the properties of the NC, which are computed every 5e3 iterations. %
The theorems are verified by demonstrating the convergence of %
the NC metrics to zero. 
We use the following metrics for measuring NC, which are similar to those in \citep{papyan2020prevalence,zhu2021geometric} but include also %
a metric for collapse
to orthogonal frames.

First, for a given set of $n $ features for each of $K$ classes, $\{ \h_{k,i} \}$, we define the per-class and global means as
$\overline{\h}_k := \frac{1}{n}\sum_{i=1}^n \h_{k,i}$ and $\overline{\h}_G := \frac{1}{Kn}\sum_{k=1}^k \sum_{i=1}^n \h_{k,i}$, respectively,
as well as the mean features matrix
$\overline{\H}:=\left [ \overline{\h}_1, \ldots, \overline{\h}_K \right ]$.
Next, we define the within-class and between-class $d \times d$ covariance matrices
\begin{align*}
 \bSigma_W &:= \frac{1}{Kn}\sum_{k=1}^K \sum_{i=1}^n (\h_{k,i}-\overline{\h}_k)(\h_{k,i}-\overline{\h}_k)^\top,  \\ 
 \bSigma_B &:=\frac{1}{K}\sum_{k=1}^K (\overline{\h}_k-\overline{\h}_G)(\overline{\h}_k-\overline{\h}_G)^\top.
\end{align*}
Now, we turn to define three metrics of NC.

{ $NC_1$ for measuring within-class variability:}
\begin{align}
\label{Eq_NC1}
NC_1 := \frac{1}{K}\tr \left ( \bSigma_W \bSigma_B^\dagger \right ),
\end{align}
where $\bSigma_B^\dagger$ denotes the pseudoinverse of $\bSigma_B$.

{ $NC_2$ for measuring the similarity of the mean features to the structured frames:}
\begin{align}
NC_2^{ETF} &:= \left \| \frac{\overline{\H}^\top \overline{\H}}{\| \overline{\H}^\top \overline{\H} \|_F} - \frac{1}{\sqrt{K-1}}(\I_K - \frac{1}{K}\1_K\1_K^\top)  \right \|_F \nonumber \\ 
\label{Eq_NC2h_of}
NC_2^{OF} &:= \left \| \frac{\overline{\H}^\top \overline{\H}}{\| \overline{\H}^\top \overline{\H} \|_F} - \frac{1}{\sqrt{K}}\I_K  \right \|_F
\end{align}
where the simplex ETFs and the OFs are normalized to unit Frobenius norm.

{ $NC_3$ for measuring the alignment of the last weights and the mean features:}
\begin{align}
\label{Eq_NC3}
NC_3 := \left \| \W/\| \W \|_F - \overline{\H}^\top /\| \overline{\H} \|_F \right \|_F.
\end{align}

Figure~\ref{fig:thm1} corroborates Theorem~\ref{thm_nc_of} for $K=4, d=20, n=50$ %
and $\lambda_W=\lambda_H=0.005$ (no bias is used, equivalently $\lambda_b \xrightarrow{} \infty$). Both $\W$ and $\H$ are initialized with standard normal distribution and are optimized with plain gradient descent with step-size 0.1.

Figure~\ref{fig:thm2} corroborates Theorem~\ref{thm_nc_simplex_etf} for $K=4, d=20, n=50$, %
$\lambda_W=\lambda_H=0.005$ and $\lambda_b=0$. All $\W$, $\H$ and $\b$ are initialized with standard normal distribution and are optimized with plain gradient descent with step-size 0.1.

Figure~\ref{fig:thm3} corroborates %
Theorem~\ref{thm_nc_of_deeper}
for $K=4, d=20, n=50$ %
and $\lambda_{W_2}=\lambda_{W_1}=\lambda_{H_1}=0.005$ (no bias is used). All $\W_2$, $\W_1$ and $\H_1$ are initialized with standard normal distribution scaled by 0.1 and are optimized with plain gradient descent with step-size 0.1.
The metrics are computed for $\W=\W_2$ and $\H=\W_1\H_1$. 
We also compute $NC_1$ and $NC_2^{OF}$ for the first layer's features $\H=\H_1$.
The collapse of both $\W_1\H_1$ and $\H_1$ to OF (demonstrated by NC1 and NC2 converging to zero) is in agreement with  %
Theorem~\ref{thm_nc_of_deeper}. 

Figure~\ref{fig:thm4} %
corroborates Theorem~\ref{thm_nc_of_deeper_nonlin}
that considers the nonlinear model in \eqref{Eq_prob_nonlin}.
We use $K=4, d=20, n=50$ and $\lambda_{W_2}=\lambda_{W_1}=\lambda_{H_1}=0.005$ (no bias is used). All $\W_2$, $\W_1$ and $\H_1$ are initialized with standard normal distribution scaled by 0.1 and are optimized with plain gradient descent with step-size 0.1.
The metrics are computed for $\W=\W_2$ and $\H=\sigma(\W_1\H_1)$.  %
We also compute $NC_1$ and $NC_2^{OF}$ for the first layer's features $\H=\H_1$ (as well as for the pre-ReLU features $\H=\W_1\H_1$).

Comparing Figures~\ref{fig:thm3} and \ref{fig:thm4} (experiments with different hyper-parameter setting yield similar results, as shown in Appendix~\ref{app:exp}), 
we observe that adding the ReLU nonlinearity to the model better distinguishes between the behavior of the features in the two levels, both in the rate of the collapse and in its structure.

Finally, we show the similarity of the NC metrics that are obtained for the {\em nonlinear} extended UFM in Figure~\ref{fig:thm4} (rather than those in Figure~\ref{fig:thm3}) and metrics obtained by a practical well-trained DNN, namely ResNet18 \citep{he2016deep} (composed of 4 ResBlocks), trained on MNIST with SGD with learning rate 0.05 (divided by 10 every 40 epochs) and weight decay ($L_2$ regularization) of 5e-4. 
Figure~\ref{fig:mnist_mse_noBias_and_CE} shows the results for two cases: 1) MSE loss without bias in the FC layer; and 2) the widely-used setting, with cross-entropy loss and bias.
(Additional experiments with CIFAR10 dataset appear in Appendix~\ref{app:exp}).
The behaviors of the metrics in both cases correlate with the one of the extended UFM in Figure~\ref{fig:thm4}.

\section{Conclusion} %

In this work, we first characterized the (global) minimizers of the unconstrained features model (UFM) for regularized MSE loss, showing some distinctions 
from the neural collapse (NC) results that have been obtained for the 
\tomtr{UFM with} 
cross-entropy loss in recent works \tomtr{(such as the effect of the bias term)}.
Then, we mitigated the inability of the plain UFM to capture any NC behavior that happens across depth by adding another layer of weights as well as ReLU nonlinearity to the model and generalized our previous results.
Finally, we empirically verified the theorems and demonstrated the usefulness of our nonlinear extended UFM in modeling the %
NC phenomenon that occurs in the training of practical networks.

The aforementioned experiments further demonstrated the necessity of the nonlinearity in the model.
We note, however, that adding a ReLU nonlinearity in the plain UFM, after the single level of features (with no additional layer of weights), is problematic. Optimizing such a model with simple gradient-based method after random initialization (which is the common way to train practical DNNs), is doomed to fail because the negative entries in the first layer cannot be modified.
The extended model that is considered in this paper does not have this limitation.

\tomtm{
As directions for future research, 
we believe that analyzing the gradient descent dynamics of the proposed extended UFM may lead to insights 
on gradient-based training of practical networks
that cannot be obtained from the dynamics of the plain UFM.
Generalizing the results that are obtained for the plain and extended UFMs to models where the features cannot be freely optimized, but are rather linked to some data distribution is also of high interest. In this front, the result in Theorem~\ref{thm_nc_of_asymp} is encouraging, though, it is only asymptotic.
When the training data is limited and the question of generalization arises (as in real-world settings), it may not be possible to show positive effects of NC %
on the generalization without departing from the plain UFM, which has limited expressiveness when the features are fixed.
On the other hand, the proposed nonlinear extended UFM seems to be more suitable for such analysis, as, in fact, it has a shallow MLP on top of the first level of features.}

\section*{Acknowledgements}

\tomtr{This work has been partially supported by the Alfred P. Sloan Foundation, NSF RI-1816753, NSF CAREER CIF-1845360, and NSF CCF-1814524.}

\bibliography{main_camera_ready_4Arxiv}
\bibliographystyle{icml2022}

\newpage
\appendix
\onecolumn

\section{Proof of Theorem~\ref{thm_nc_of}}
\label{app:proof1}

\begin{proof}

The proof is based on lower bounding $f(\W,\H):= \frac{1}{2N} \| \W \H - \Y \|_F^2 + \frac{\lambda_W}{2} \|\W\|_F^2 + \frac{\lambda_H}{2} \|\H\|_F^2$ by a sequence of inequalities that hold with equality if and only if \eqref{Eq_thm_1}-\eqref{Eq_thm_4} are satisfied. First, observe that
\begin{align}
\label{Eq_poof_1}
&\frac{1}{2N} \| \W \H - \Y  \|_F^2 + \frac{\lambda_W}{2} \|\W\|_F^2 + \frac{\lambda_H}{2} \|\H\|_F^2 \\ \nonumber
    &= \frac{1}{2Kn} \sum_{k=1}^{K} \sum_{i=1}^{n} \| \W \h_{k,i} - \y_k \|_2^2   %
+ \frac{\lambda_W}{2} \sum_{k=1}^{K} \|\w_{k}\|_2^2 + \frac{\lambda_H}{2} \sum_{k=1}^{K} \sum_{i=1}^{n} \|\h_{k,i}\|_2^2 \\ \nonumber
    &\stackrel{(a)}{\geq} \frac{1}{2Kn} \sum_{k=1}^{K} n \frac{1}{n} \sum_{i=1}^{n} ( \w_k^\top \h_{k,i} - 1 )^2 %
+ \frac{\lambda_W}{2} \sum_{k=1}^{K} \|\w_{k}\|_2^2 + \frac{\lambda_H}{2} \sum_{k=1}^{K} n \frac{1}{n} \sum_{i=1}^{n} \|\h_{k,i}\|_2^2 \\ \nonumber 
    &\stackrel{(b)}{\geq} \frac{1}{2Kn} \sum_{k=1}^{K} n \left ( \w_k^\top \frac{1}{n} \sum_{i=1}^{n} \h_{k,i} - 1 \right )^2 %
+ \frac{\lambda_W}{2} \sum_{k=1}^{K} \|\w_{k}\|_2^2 + \frac{\lambda_H}{2} \sum_{k=1}^{K} n  \left \| \frac{1}{n} \sum_{i=1}^{n} \h_{k,i} \right \|_2^2 \\ \nonumber     
\end{align}
The inequality $(a)$ follows from ignoring all the entries except $k$ in the $K\times1$ vector $\W \h_{k,i} - \y_k$, and holds with equality iff $\w_{k'}^\top \h_{k,i}=0$ for all $k' \neq k$ and $i \in [n]$.
In $(b)$ we used Jensen's inequality, which (due to the strict convexity of $(\cdot-1)^2$ and $\|\cdot\|^2$) holds with equality iff $\h_{k,1}=\ldots=\h_{k,n}$ for all $k \in [K]$.
Indeed, note that the equality condition for $(b)$ is satisfied by \eqref{Eq_thm_1}, and the equality condition for $(a)$ is a consequence of \eqref{Eq_thm_1}, \eqref{Eq_thm_3} and \eqref{Eq_thm_4} (which yield \eqref{Eq_thm_6}).

Next, to simplify the notation, let us denote $\h_k := \frac{1}{n} \sum_{i=1}^{n} \h_{k,i}$. Thus, continuing from the last RHS in \eqref{Eq_poof_1}, we have
\begin{align}
\label{Eq_poof_2}
    &\frac{1}{2K} \sum_{k=1}^{K} \left ( \w_k^\top \h_k - 1 \right )^2  %
+ \frac{\lambda_W}{2} K\frac{1}{K} \sum_{k=1}^{K} \|\w_{k}\|_2^2 + \frac{n\lambda_H}{2} K\frac{1}{K} \sum_{k=1}^{K} \left \| \h_k \right \|_2^2 \\ \nonumber     
    &\stackrel{(c)}{\geq} \frac{1}{2} \left ( \frac{1}{K} \sum_{k=1}^{K} \w_k^\top \h_k - 1 \right )^2  %
+ \frac{\lambda_W}{2} K \left ( \frac{1}{K} \sum_{k=1}^{K} \|\w_{k}\|_2 \right )^2 + \frac{n\lambda_H}{2} K \left( \frac{1}{K} \sum_{k=1}^{K} \left \| \h_k \right \|_2 \right)^2 \\ \nonumber  
    &\stackrel{(d)}{\geq} \frac{1}{2} \left ( \frac{1}{K} \sum_{k=1}^{K} \w_k^\top \h_k - 1 \right )^2  %
+ K \sqrt{n \lambda_H \lambda_W} \left ( \frac{1}{K} \sum_{k=1}^{K} \|\w_{k}\|_2 \right ) \left( \frac{1}{K} \sum_{k=1}^{K} \left \| \h_k \right \|_2 \right) %
\end{align}
In $(c)$ we used Jensen's inequality, which holds with equality iff 
\begin{align*}
      &\w_{1}^\top \h_{1} = \ldots = \w_{K}^\top \h_{K}, \\
      &\| \w_{1} \|_2  = \ldots = \| \w_{K} \|_2, \\
      &\| \h_{1} \|_2  = \ldots = \| \h_{K} \|_2, 
\end{align*}
which are satisfied when conditions \eqref{Eq_thm_2} and \eqref{Eq_thm_4} are satisfied.
In $(d)$ we used the AM–GM inequality, i.e., $\frac{a}{2}+\frac{b}{2} \geq \sqrt{ab}$, with $a = \lambda_W \left ( \frac{1}{K} \sum_{k=1}^{K} \|\w_{k}\|_2 \right )^2$ and $b = n\lambda_H \left( \frac{1}{K} \sum_{k=1}^{K} \left \| \h_k \right \|_2 \right)^2$. It holds with equality iff $a=b$, which is satisfied by \eqref{Eq_thm_4} that implies $\lambda_W \|\w_k\|_2^2 = n \lambda_H \|\h_k\|_2^2$.

Note that so far all the iff conditions are satisfied by both $(\W^*,\H^*)$ that satisfy \eqref{Eq_thm_1}-\eqref{Eq_thm_4} and the trivial $(\W^*,\H^*)=(\0,\0)$. 
Now, it is left to show that if $K \sqrt{n \lambda_H \lambda_W} \leq 1$
then $\w_k$ and $\h_k$ must have the same direction, as implied by \eqref{Eq_thm_4}, which will also yield the orthogonality of $\{\h^*_k\}$ and $\{\w^*_k\}$.
While for $K \sqrt{n \lambda_H \lambda_W} > 1$, we get the zero minimizer.

As all the inequalities (a)-(d) are attainable with iff conditions, we can consider now $(\W,\H)$ that satisfy these conditions to further lower the bound. Specifically, using the symmetry w.r.t. $k$, the last RHS in \eqref{Eq_poof_2} turns into the expression 
\begin{align}
&\frac{1}{2} \left ( \w_k^\top \h_k -1 \right )^2 + K\sqrt{n\lambda_H\lambda_W}\|\w_k\|_2\|\h_k\|_2 \\ \nonumber
&= \frac{1}{2} \left ( \|\w_k\|_2\|\h_k\|_2 \mathrm{cos}{\alpha} -1 \right )^2 + K\sqrt{n\lambda_H\lambda_W}\|\w_k\|_2\|\h_k\|_2,
\end{align}
where $\alpha$ is the angle between $\w_k$ and $\h_k$.
Invoking Lemma~\ref{lemma_aux} with $\beta = \|\w_k\|_2\|\h_k\|_2$ and $c=K\sqrt{n\lambda_H\lambda_W}$, we get that if $K\sqrt{n\lambda_H\lambda_W}>1$ then the minimizer is $(\W^*,\H^*)=(\0,\0)$ (since $\|\w_k\|_2\|\h_k\|_2=0$), and otherwise, the minimizer must obey $\alpha=0$.
Therefore, we get the desired result that $\w_k^*$ and $\h_k^*$ must have the same direction.
Together with $\lambda_W \|\w_k\|_2^2 = n \lambda_H \|\h_k\|_2^2$ (which is required to attain equality for AM-GM), we get the necessity %
of $\w_k^* = \sqrt{n \lambda_H / \lambda_W} \h_{k}^*$ in \eqref{Eq_thm_4}. 
Finally, the orthogonality of $\{\h^*_k\}$ (and similarly of ${\w^*_k}$) follows from
$$
\h_{k'}^{*\top} \h_{k}^* = \frac{1}{\sqrt{n \lambda_H / \lambda_W}} \w_{k'}^{*\top} \h_{k}^* = 0 \,\,\,\, \forall k'\neq k
$$
where we used the previous condition $\w_{k'}^{*\top} \h_{k}^*=0$ for all $k' \neq k$, which is necessary to attain equality in \eqref{Eq_poof_1}.

\end{proof}

\begin{lemma}
\label{lemma_aux}
Let
\begin{align}
\label{Eq_lemma}
\tilde{f}(\alpha,\beta) = \frac{1}{2} \left ( \beta \mathrm{cos}{\alpha} -1 \right )^2 + c \beta,
\end{align}
where $\beta \geq 0$ and $c>0$. Then, (i) if $c>1$ then $\tilde{f}$ is minimized by $\beta^*=0$ and the minimal value is $\frac{1}{2}$; (ii) if $c \leq 1$ then $\tilde{f}$ is minimized by $(\alpha^*,\beta^*)=(0,1-c)$ and the minimal value is $c-\frac{1}{2}c^2$.
\end{lemma}

\begin{proof}
The proof is based on separately analyzing the cases $\beta=0$, $0<\beta<1$ and $\beta \geq 1$.

For $\beta=0$, we get objective value of $\frac{1}{2}$ for any $\alpha$. 
Assuming that $0<\beta<1$, clearly, the minimizer of \eqref{Eq_lemma} w.r.t. $\alpha$ is only $\alpha^*=0$ (or other integer multiplications of $2\pi$). Thus, we have
$$
\tilde{f}(0,\beta) = \frac{1}{2} \left ( \beta -1 \right )^2 + c \beta = \frac{1}{2} \beta^2 - (1-c)\beta + \frac{1}{2},
$$
which is a ``smiling" parabola in $\beta$, with feasible minimum at $\beta^*=\mathrm{max}\{1-c, 0\}$.
This means that if $c>1$ we get the (feasible) minimum at $(\alpha^*,\beta^*)=(0,0)$, for which $\tilde{f}(\alpha^*,\beta^*)=\frac{1}{2}$. 
If $c\leq1$, we get minimum at $(\alpha^*,\beta^*)=(0,1-c)$ with objective value of $\tilde{f}(\alpha^*,\beta^*)=\frac{1}{2}c^2 +c(1-c) = c - \frac{1}{2}c^2$.

Assuming that $\beta \geq 1$, the first term in \eqref{Eq_lemma} is minimized (eliminated) by $\alpha^*=\mathrm{arccos}(1/\beta)$. Thus, we get
$
\tilde{f}(\alpha^*,\beta) = c \beta,
$
which is minimized by $\beta^*=1$, and the objective value is $\tilde{f}(\alpha^*,\beta^*)=c$.
Since $c>0$, note that this value is always larger than the minimal value obtained for $\beta < 1$.

To summarize, (i) if $c>1$ %
we get the minimizers $\tilde{f}(\alpha^*,\beta^*=0)=\frac{1}{2}$; (ii) 
If $c \leq 1$ %
we get the minimizer $\tilde{f}(\alpha^*=0,\beta^*=1-c)=c-\frac{1}{2}c^2$. %

\end{proof}

\newpage
\subsection{Alternative proof for Theorem~\ref{thm_nc_of}}
\label{app:proof1_alt}

We present here an alternative proof for Theorem~\ref{thm_nc_of}. 
\tomr{The strategy of this proof is more similar (than the preceding proof)} 
to the one we take to handle the three layer case in Appendix~\ref{app:proof3}. 
\tomr{(Steps in the previous version of this appendix, which can be justified with alignment of singular-bases similarly to Appendix A in \cite{arora2018optimization}, are replaced by more elementary 
and simple arguments).} 

\tomr{
We start by showing the zero within-class variability property of a minimizer $\H$ similarly to the previous proof:
\begin{align}
\label{Eq_poof_1_alt}
&\frac{1}{2N} \| \W \H - \Y  \|_F^2 + \frac{\lambda_W}{2} \|\W\|_F^2 + \frac{\lambda_H}{2} \|\H\|_F^2 \\ \nonumber
    &= \frac{1}{2Kn} \sum_{k=1}^{K} n\frac{1}{n} \sum_{i=1}^{n} \| \W \h_{k,i} - \y_k \|_2^2 
+ \frac{\lambda_W}{2} \|\W\|_F^2 + \frac{\lambda_H}{2} \sum_{k=1}^{K} n\frac{1}{n} \sum_{i=1}^{n} \|\h_{k,i}\|_2^2 \\ \nonumber
    &\geq \frac{1}{2Kn} \sum_{k=1}^{K} n \left \| \W \frac{1}{n} \sum_{i=1}^{n} \h_{k,i} - \y_k \right \|^2 %
+ \frac{\lambda_W}{2} \|\W\|_F^2 + \frac{\lambda_H}{2} \sum_{k=1}^{K} n  \left \| \frac{1}{n} \sum_{i=1}^{n} \h_{k,i} \right \|_2^2,    
\end{align}
where 
we used Jensen's inequality, which (due to the strict convexity of %
$\|\cdot\|^2$) holds with equality iff $\h_{k,1}=\ldots=\h_{k,n}$ for all $k \in [K]$.
That is, the minimizer must exhibit zero within-class variability: $\H = \overline{\H} \otimes \1_n^\top$ for some $\overline{\H} \in \mathbb{R}^{d \times K}$.} %

\tomr{We proceed
by computing the gradients of the objective (after substituting $\H = \overline{\H} \otimes \1_n^\top$ and $\Y=\I_K \otimes \1_n^\top$): 
$f(\W, \overline{\H} ):= \frac{1}{2K} \| \W \overline{\H} - \I_K \|_F^2 + \frac{\lambda_W}{2} \|\W\|_F^2 + \frac{n\lambda_H}{2} \| \overline{\H} \|_F^2$
\begin{align}
\frac{\partial f}{\partial \overline{\H}} &= \W^\top \frac{1}{K} ( \W \overline{\H} - \Y) + n\lambda_{H} \overline{\H}, \\
\frac{\partial f}{\partial \W} &= \frac{1}{K} ( \W \overline{\H} - \I_K) \overline{\H}^\top + \lambda_{W} \W. 
\end{align}
From $\frac{\partial f}{\partial \W}=\0$, observe that the minimizer w.r.t.~$\W$ is a closed-form function of $\overline{\H}$:
\begin{align}
\label{app1_Eq_W_func_H}
\W(\overline{\H}) = \overline{\H}^\top ( \overline{\H} \overline{\H}^\top + K \lambda_{W} \I_d )^{-1}.
\end{align}
(Similarly, from $\frac{\partial f}{\partial \overline{\H}} = \0$ we can express the minimizer w.r.t.~$\overline{\H}$ as a function of $\W$.)}

\tomr{
Let us denote the compact SVD of $\overline{\H}=\U\S\V^\top$, where $\U \in \mathbb{R}^{d \times K}$ is a partial orthonormal matrix and $\V \in \mathbb{R}^{K \times K}$ is an orthonormal matrix and $\S \in \mathbb{R}^{K \times K}$ is a diagonal matrix with $\{ s_k \}_{k=1}^K$ on its diagonal. 
From \eqref{app1_Eq_W_func_H} %
we have that a minimizer $(\W,\overline{\H})$ obeys
\begin{align*}
    &\W %
    = \V \S \U^\top (\U \S^2 \U^\top + K\lambda_{W} \I_d )^{-1} = \V \S (\S^2 + K\lambda_{W} \I_K )^{-1} \U^\top, \\
    &\W \overline{\H} - \I_K = \V \left ( \S (\S^2 + K\lambda_{W}\I_d)^{-1} \S - \I_K \right ) \V^\top = \V \mathrm{diag} \left \{ \frac{-K\lambda_{W}}{(s_1^2 + K\lambda_{W})} , ... , \frac{-K\lambda_{W}}{(s_K^2 + K\lambda_{W})} \right \} \V^\top.
\end{align*}
Substituting these expressions in the objective, and using the fact that unitary operators do not change the Frobenius norm, we can express the objective as a function of the singular values $\{s_k\}$
\begin{align}
\label{Eq_poof_2_alt}
f(\W, \overline{\H}) &= \frac{1}{2K} \sum_{k=1}^K \frac{(K\lambda_{W})^2}{(s_k^2 + K\lambda_{W})^2} + \frac{\lambda_{W}}{2} \sum_{k=1}^K \frac{s_k^2}{(s_k^2 + K\lambda_{W})^{2}} + \frac{n\lambda_{H}}{2} \sum_{k=1}^K s_k^2 \\ \nonumber
&= \sum_{k=1}^K \left ( \frac{\lambda_W}{2} \frac{1}{s_k^2 + K\lambda_{W}} + \frac{n\lambda_{H}}{2} s_k^2 \right ).
\end{align}
At this point, we already see that the objective is separable w.r.t.~the singular values, which implies that the minimizer obeys $s_1 = ... = s_K =:s$.
Therefore, $\overline{\H} = s \U \V^\top$ and $\W = \frac{s}{s^2+K\lambda_W} \V \U^\top \propto \overline{\H}^\top$.}

\tomr{
The flat spectrum of $\overline{\H}$ implies that $\overline{\H}^\top \overline{\H} \propto \I_K$, since
\begin{align}
\overline{\H}^\top \overline{\H} = s^2 \V \U^\top \U \V^\top = s^2 \V \V^\top = s^2 \I_K.
\end{align}
Similarly, $\W\overline{\H} \propto \I_K$ %
and $\W\W^\top \propto \I_K$. 
Denoting $\sigma_{\overline{H}}:=s$ and $\sigma_{W}:=\frac{s}{s^2+K\lambda_W}$, 
the above results imply that %
a minimizer of \eqref{Eq_prob} %
is given by $(\W,\overline{\H} \otimes \1_n^\top)$ where
\begin{align}
\label{Eq_poof_3_alt}
    \W &= \sigma_{W} \R^\top \in \mathbb{R}^{K \times d} \\
\label{Eq_poof_4_alt}    
    \overline{\H} &= \sigma_{\overline{H}} \R \in \mathbb{R}^{d \times K}
\end{align}
with arbitrary orthonormal matrix $\R \in \mathbb{R}^{d \times K}$ ($\R^\top\R = \I_K$).}

\tomr{
We can compute $s$ by minimization of one term of \eqref{Eq_poof_2_alt}. Alternatively, the values of $\sigma_W$ and $\sigma_{\overline{H}}$ can be determined by minimizing the simplified objective that resembles the one in the previous appendix (obtained by substituting \eqref{Eq_poof_3_alt}-\eqref{Eq_poof_4_alt}):}
\begin{align}
\label{app1_obj_W_Hbar3}
f(\W, \overline{\H}) &= \frac{1}{2} ( \sigma_{{W}} \sigma_{\overline{H}} - 1 )^2 + K \frac{\lambda_W}{2} \sigma_W^2 + K \frac{n\lambda_H}{2} \sigma_{\overline{H}}^2.
\end{align}
The derivatives are given by
\begin{align}
\label{app1_obj_W_Hbar4}
\frac{\partial}{\partial \sigma_W} f &= \sigma_{\overline{H}} ( \sigma_{{W}} \sigma_{\overline{H}} - 1 ) + K \lambda_W \sigma_W = 0, \\
\frac{\partial}{\partial \sigma_{\overline{H}}} f &= \sigma_{W} ( \sigma_{{W}} \sigma_{\overline{H}} - 1 ) + K n \lambda_H \sigma_{\overline{H}} = 0,
\end{align}
implying that $\lambda_W \sigma_W^2 = n \lambda_H \sigma_{\overline{H}}^2$, which can also be obtained by attaining the AM-GM inequality
$$
K \frac{\lambda_W}{2} \sigma_W^2 + K \frac{n\lambda_H}{2} \sigma_{\overline{H}}^2 \geq K \sqrt{n\lambda_H\lambda_W} \sigma_{{W}} \sigma_{\overline{H}}.
$$
Therefore, setting $\beta=\sigma_{{W}} \sigma_{\overline{H}}$, to find the eigenvalues of the minimizers we just need to find $\beta \geq 0$ that minimizes
\begin{align}
\tilde{f}(\beta) = \frac{1}{2} \left ( \beta -1 \right )^2 + c \beta,
\end{align}
for $c = K \sqrt{n\lambda_H\lambda_W} > 0$.
It can be shown that: (i) if $c>1$ then $\tilde{f}$ is minimized by $\beta^*=0$ and the minimal value is $\frac{1}{2}$; (ii) if $c \leq 1$ then $\tilde{f}$ is minimized by $\beta^*=1-c$ and the minimal value is $c-\frac{1}{2}c^2$.

Summarizing our finding, we have that if $c = K \sqrt{n\lambda_H\lambda_W} > 1$ then the minimizer is $(\W,\H)=(\0,\0)$ (because the singular values of the matrices are zero).
On the other hand, if $c = K \sqrt{n\lambda_H\lambda_W} \leq 1$ then the minimizers obey $\H = \overline{\H} \otimes \1_n^\top$, $\W = \sqrt{\frac{n\lambda_H}{\lambda_W}}\overline{\H}^\top$, and 
$$
\W \overline{\H} = \sigma_W \sigma_{\overline{H}} \I_K = (1-c) \I_K
$$
$$
\overline{\H}^\top \overline{\H} = \sigma_{\overline{H}}^2 \I_K = (1-c) \sqrt{\frac{\lambda_W}{n \lambda_H}} \I_K
$$
$$
\W \W^\top = \sigma_{W}^2 \I_K = (1-c) \sqrt{ \frac{n \lambda_H}{\lambda_W} } \I_K
$$
as stated in the theorem.

\newpage

\section{Proof of Theorem~\ref{thm_nc_simplex_etf}}
\label{app:proof2}

\begin{proof}

First, note that the objective %
\begin{align}
\label{app22_Eq_poof_1}
f(\W,\H,\b) := \frac{1}{2N} \| \W \H + \b \1_N^\top - \Y \|_F^2 + \frac{\lambda_W}{2} \|\W\|_F^2 + \frac{\lambda_H}{2} \|\H\|_F^2
\end{align}
is convex w.r.t. $\b$, for which there is the following closed-form minimizer (which depends on $\W\H$)
\begin{align}
\label{app22_Eq_poof_2}
\b^* = \frac{1}{N} \left ( \Y - \W\H \right ) \1_N = \frac{1}{N} \sum_{k=1}^K \sum_{i=1}^n ( \y_k - \W \h_{k,i} ).
\end{align}
Since $\{ \y_k \}$ are one-hot vectors, note that for $k' \in [K]$
\begin{align}
\label{app22_Eq_poof_b}
b^*_{k'} = \frac{n}{N} -\frac{1}{N} \sum_{k=1}^K \sum_{i=1}^n \w_{k'}^\top \h_{k,i} = \frac{1}{K} - \w_{k'}^\top\h_G,
\end{align}
where $\h_G:= \frac{1}{N} \sum_{k=1}^K \sum_{i=1}^n \h_{k,i}$ is the global feature mean. 

The proof is based on lower bounding $f(\W,\H,\b^*)$ by a sequence of inequalities that hold with equality if and only if \eqref{Eq_thm2_0}-\eqref{Eq_thm2_4} are satisfied. Observe that
\begin{align}
\label{app2_Eq_poof_1}
    &\frac{1}{2Kn} \sum_{k=1}^{K} \sum_{i=1}^{n} \| \W \h_{k,i} + \b^* - \y_k \|_2^2   %
+ \frac{\lambda_W}{2} \sum_{k=1}^{K} \|\w_{k}\|_2^2 + \frac{\lambda_H}{2} \sum_{k=1}^{K} \sum_{i=1}^{n} \|\h_{k,i}\|_2^2 %
\\ \nonumber
    &= \frac{1}{2K} \sum_{k'=1}^{K} \sum_{k=1}^{K} \frac{1}{n} \sum_{i=1}^{n} ( \w_{k'}^\top (\h_{k,i} - \h_G) + \frac{1}{K} - 1_{k'=k} )^2   %
+ \frac{\lambda_W}{2} \sum_{k=1}^{K} \|\w_{k}\|_2^2 + \frac{\lambda_H}{2} \sum_{k=1}^{K} n \frac{1}{n}\sum_{i=1}^{n} \|\h_{k,i}\|_2^2 %
\\ \nonumber
    &\stackrel{(b)}{\geq} \frac{1}{2K} \sum_{k'=1}^{K} \sum_{k=1}^{K} \left ( \w_{k'}^\top (\frac{1}{n} \sum_{i=1}^{n} \h_{k,i} - \h_G) + \frac{1}{K} - 1_{k'=k} \right )^2 %
+ \frac{\lambda_W}{2} \sum_{k=1}^{K} \|\w_{k}\|_2^2 + \frac{\lambda_H}{2} \sum_{k=1}^{K} n  \left \| \frac{1}{n} \sum_{i=1}^{n} \h_{k,i} \right \|_2^2 %
\\ \nonumber     
\end{align}
In $(b)$ we used Jensen's inequality, which (due to the strict convexity of $\|\cdot\|^2$) holds with equality iff $\h_{k,1}=\ldots=\h_{k,n}$ for all $k \in [K]$.
Indeed, note that the equality condition for $(b)$ is satisfied by \eqref{Eq_thm2_1}.

Next, to simplify the notation, let us denote $\h_k := \frac{1}{n} \sum_{i=1}^{n} \h_{k,i}$ (note that $\h_G=\frac{1}{K} \sum_{k=1}^K \h_k$). Thus, continuing from the last RHS in \eqref{app2_Eq_poof_1}, we have
\begin{align*}
&\frac{1}{2K} \sum_{k'=1}^{K} \sum_{k=1}^{K} \left ( \w_{k'}^\top (\h_k - \h_G) + \frac{1}{K} - 1_{k'=k} \right )^2 
+ \frac{\lambda_W}{2} \sum_{k=1}^{K} \|\w_{k}\|_2^2 + \frac{n \lambda_H}{2}  \sum_{k=1}^{K}  \left \| \h_k \right \|_2^2 %
\\ \nonumber 
&= \frac{1}{2K} \sum_{k=1}^{K} \left ( \w_{k}^\top (\h_k - \h_G) - \frac{K-1}{K} \right )^2  
+ \frac{K-1}{2K} \sum_{k'=1}^{K} \frac{1}{K-1} \sum_{k=1, k \neq k'}^{K} \left ( \w_{k'}^\top (\h_k - \h_G) + \frac{1}{K} \right )^2 
 \\ \nonumber 
& \hspace{2mm} + \frac{\lambda_W}{2} K\frac{1}{K} \sum_{k=1}^{K} \|\w_{k}\|_2^2 + \frac{n \lambda_H}{2} K\frac{1}{K} \sum_{k=1}^{K}  \left \| \h_k \right \|_2^2 %
\\ \nonumber
\end{align*}
\begin{align}
\label{app2_Eq_poof_2}
    &\stackrel{(c)}{\geq} \frac{1}{2} \left ( \frac{1}{K} \sum_{k=1}^{K} \w_{k}^\top (\h_k - \h_G) - \frac{K-1}{K} \right )^2  
+ \frac{K-1}{2K} \sum_{k'=1}^{K} \left ( \frac{1}{K-1} \sum_{k=1, k \neq k'}^{K} \w_{k'}^\top (\h_k - \h_G) + \frac{1}{K} \right )^2
 \\ \nonumber 
    & \hspace{2mm} 
+ \frac{\lambda_W}{2} K \left ( \frac{1}{K} \sum_{k=1}^{K} \|\w_{k}\|_2 \right )^2 + \frac{n\lambda_H}{2} K \left( \frac{1}{K} \sum_{k=1}^{K} \left \| \h_k \right \|_2 \right)^2  %
\\ \nonumber  
    &\stackrel{(d)}{\geq} \frac{1}{2} \left ( \frac{1}{K} \sum_{k=1}^{K} \w_{k}^\top (\h_k - \h_G) - \frac{K-1}{K} \right )^2  
+ \frac{K-1}{2K} \sum_{k'=1}^{K} \left ( \frac{1}{K-1} \sum_{k=1, k \neq k'}^{K} \w_{k'}^\top (\h_k - \h_G) + \frac{1}{K} \right )^2
 \\ \nonumber 
    & \hspace{2mm} 
+ K \sqrt{n \lambda_H \lambda_W} \left ( \frac{1}{K} \sum_{k=1}^{K} \|\w_{k}\|_2 \right ) \left( \frac{1}{K} \sum_{k=1}^{K} \left \| \h_k \right \|_2 \right) %
\end{align}
In $(c)$ we used Jensen's inequality, which holds with equality iff 
\begin{align}
\label{app2_Eq_poof_2b_1}
      &\w_{1}^\top ( \h_{1} - \h_G ) = \ldots = \w_{K}^\top ( \h_{K} - \h_G), \\
\label{app2_Eq_poof_2b_2}
      &\w_{k'}^\top ( \h_{k_1} - \h_G ) = \w_{k'}^\top ( \h_{k_2} - \h_G), \,\,\, \forall k_1,k_2 \in [K] \setminus k', \\
\label{app2_Eq_poof_2b_3}
      &\| \w_{1} \|_2  = \ldots = \| \w_{K} \|_2, \\
\label{app2_Eq_poof_2b_4}
      &\| \h_{1} \|_2  = \ldots = \| \h_{K} \|_2,
\end{align}
which are satisfied when conditions \eqref{Eq_thm2_2} and \eqref{Eq_thm2_4} are satisfied.
In $(d)$ we used the AM–GM inequality, i.e., $\frac{a}{2}+\frac{b}{2} \geq \sqrt{ab}$, with $a = \lambda_W \left ( \frac{1}{K} \sum_{k=1}^{K} \|\w_{k}\|_2 \right )^2$ and $b = n\lambda_H \left( \frac{1}{K} \sum_{k=1}^{K} \left \| \h_k \right \|_2 \right)^2$. It holds with equality iff $a=b$, which is satisfied by \eqref{Eq_thm2_4} that implies $\lambda_W \|\w_k\|_2^2 = n \lambda_H \|\h_k\|_2^2$.
Now, observe that the first two terms in the last RHS of \eqref{app2_Eq_poof_2} are invariant to the global mean of $\H$ (since it is subtracted there from $\{ \h_k \}$). Therefore, the expression can be further reduced by requiring that $\h_G$ minimizes the term $\frac{1}{K} \sum_{k=1}^K \|\h_k\|_2$. 
To this end, using the triangle inequality $\|\h_k\|_2 \geq \|\h_k - \h_G \|_2 - \| \h_G \|_2$, we have
$$
\frac{1}{K} \sum_{k=1}^K \|\h_k\|_2 \geq \frac{1}{K} \sum_{k=1}^K \|\h_k - \h_G \|_2 - \| \h_G \|_2,
$$
which becomes equality when $\h_G=\0$, as required by condition \eqref{Eq_thm2_1g}.
From \eqref{app22_Eq_poof_b}, this also implies that $\b^*=\frac{1}{K}\1_K$, as required by condition \eqref{Eq_thm2_0}.
Next, consider $\w_{k'}^\top \h_{k} = \| \w_{k'} \|_2 \| \h_{k} \|_2 \mathrm{cos}\tilde{\alpha}_{k',k}$, where $\tilde{\alpha}_{k,k'}$ denotes the angle between $\w_{k'}$ and $\h_{k}$. From \eqref{app2_Eq_poof_2b_2}-\eqref{app2_Eq_poof_2b_4} it follows that $\tilde{\alpha}_{k',k}$ is exactly the same for any chosen $k' \in [K]$ and $k\in [K] \setminus k'$. This equiangular property implies that
the minimal (most negative) possible value of $\mathrm{cos}\tilde{\alpha}_{k',k}$ is given by  
 $\mathrm{cos}\tilde{\alpha}_{k',k} = -\frac{1}{K-1}$, as we have in the standard simplex ETF (Definition \ref{def:simplex_etf}).%

Note that so far all the iff conditions are satisfied by both $(\W^*,\H^*,\b^*=\frac{1}{K}\1_K)$ that satisfy \eqref{Eq_thm2_1}-\eqref{Eq_thm2_4} and the naive $(\W^*,\H^*,\b^*)=(\0,\0,\frac{1}{K}\1_K)$. 
Now, it is left to show that if $K \sqrt{n \lambda_H \lambda_W} \leq 1$
then $\w_k$ and $\h_k$ must have the same direction, as implied by \eqref{Eq_thm2_4}, %
and the simplex 
equiangular property of $\{\h^*_k\}$ and $\{\w^*_k\}$.
While for $K \sqrt{n \lambda_H \lambda_W} > 1$, we get the naive minimizer.

As all the inequalities used so far %
are attainable with iff conditions, we can consider now $(\W,\H)$ that satisfy these conditions to further lower the bound. Specifically, using the symmetry w.r.t. $k$, and choosing any $k'\neq k$, the last RHS in \eqref{app2_Eq_poof_2} (with the required $\h_G=\0$) turns into the expression 
\begin{align}
&\frac{1}{2} \left ( \w_k^\top \h_k - \frac{K-1}{K} \right )^2 
+ \frac{(K-1)}{2} \left ( \w_{k'}^\top \h_k + \frac{1}{K} \right )^2
+ K\sqrt{n\lambda_H\lambda_W}\|\w_k\|_2\|\h_k\|_2 %
\\ \nonumber
&= \frac{1}{2} \left ( \|\w_k\|_2\|\h_k\|_2 \mathrm{cos}{\alpha}  - \frac{K-1}{K}  \right )^2 
+ \frac{(K-1)}{2} \left ( \|\w_k\|_2\|\h_k\|_2 \mathrm{cos}{\tilde{\alpha}} + \frac{1}{K} \right )^2
+ K\sqrt{n\lambda_H\lambda_W}\|\w_k\|_2\|\h_k\|_2,  %
\\ \nonumber
\end{align}
where %
we used $\alpha$ (resp. $\tilde{\alpha}$) to denote the angle between $\w_k$ and $\h_k$ (resp. $\w_k'$ and $\h_k$), and %
the necessary condition that $\|\w_k\|_2=\|\w_k'\|_2$. %

Invoking Lemma~\ref{app2_lemma_aux} with $\beta = \|\w_k\|_2\|\h_k\|_2$ and $c=K\sqrt{n\lambda_H\lambda_W}$, we get that if $K\sqrt{n\lambda_H\lambda_W}>1$ then the minimizer is $(\W^*,\H^*)=(\0,\0)$ (since $\|\w_k\|_2\|\h_k\|_2=0$), and otherwise, the minimizer must obey $\alpha=0$ and $\tilde{\alpha}=\mathrm{arccos}(-\frac{1}{K-1})$. 
Therefore, we get the desired results that $\w_k^*$ and $\h_k^*$ must have the same direction %
and $\w_{k'}^\top \h_{k} %
= - \| \w_{k} \|_2 \| \h_{k} \|_2 \frac{1}{K-1}$ for any $k' \in [K]$ and $k\in [K] \setminus k'$. 
Together with $\lambda_W \|\w_k\|_2^2 = n \lambda_H \|\h_k\|_2^2$ (which is required to attain equality for AM-GM), we get the necessity %
of $\w_k^* = \sqrt{n \lambda_H / \lambda_W} \h_{k}^*$ in \eqref{Eq_thm2_4}. 
Finally, the simplex equiangular property of $\{\h^*_k\}$ (and similarly of ${\w^*_k}$) follows from
$$
\h_{k'}^{*\top} \h_{k}^* = \sqrt{\frac{\lambda_W}{n \lambda_H}} \w_{k'}^{*\top} \h_{k}^* 
= \sqrt{\frac{\lambda_W}{n \lambda_H}} \| \w_{k'}^*\|_2 \|\h_{k}^*\|_2 \mathrm{cos}\tilde{\alpha}_{k',k} 
= \|\h_{k}^*\|_2^2 \mathrm{cos}\tilde{\alpha}_{k',k}
= - \|\h_{k}^*\|_2^2 \frac{1}{K-1}
\,\,\,\, \forall k'\neq k
$$
where we used the simplex equiangular condition between $\w_{k'}$ and $\h_{k}$ ($k' \neq k$). %

\end{proof}

\begin{lemma}
\label{app2_lemma_aux}
Let %
\begin{align}
\label{app2_Eq_lemma}
\tilde{f}(\alpha,\tilde{\alpha},\beta) = \frac{1}{2} \left ( \beta \mathrm{cos}{\alpha} - \frac{K-1}{K} \right )^2 
+ \frac{(K-1)}{2} \left ( \beta \mathrm{cos}{\tilde{\alpha}} + \frac{1}{K} \right )^2
+ c \beta,
\end{align}
where $\beta \geq 0$, $-\frac{1}{K-1} \leq \mathrm{cos}{\tilde{\alpha}} \leq 1$ and $c>0$. Then, (i) if $c>1$ then $\tilde{f}$ is minimized by $\beta^*=0$ and the minimal value is $\frac{K-1}{2K}$; (ii) if $c \leq 1$ then $\tilde{f}$ is minimized by $(\alpha^*, \tilde{\alpha}^*, \beta^*)=(0, \mathrm{arccos}(-\frac{1}{K-1}), \frac{(1-c)(K-1)}{K})$ and the minimal value is $\frac{K-1}{K} \left ( c - \frac{1}{2}c^2 \right )$.
\end{lemma}

\begin{proof}

The proof is based on separately analyzing the cases $\beta=0$, $0<\beta < \frac{K-1}{K}$ and $\beta \geq \frac{K-1}{K}$.

For $\beta=0$, we get objective value of $\frac{(K-1)^2}{2K^2} + \frac{K-1}{2K^2} = \frac{K-1}{2K}$ for any $\alpha$ and $\tilde{\alpha}$. 
Assuming that $0<\beta< \frac{K-1}{K}$, clearly, the minimizer of \eqref{app2_Eq_lemma} w.r.t. $\alpha$ is only $\alpha^*=0$ (or other integer multiplications of $2\pi$),
and the minimizer of \eqref{app2_Eq_lemma} w.r.t. $\tilde{\alpha}$ is $\tilde{\alpha}^*=\mathrm{arccos}(-\frac{1}{K-1})$ (recall the assumption $-\frac{1}{K-1} \leq \mathrm{cos}{\tilde{\alpha}} \leq 1$). 
Thus, we have
\begin{align}
\tilde{f}(0,\mathrm{arccos}(\frac{-1}{K-1}),\beta) &= \frac{1}{2} \left ( \beta - \frac{K-1}{K} \right )^2 + \frac{(K-1)}{2} \left ( -\frac{\beta}{K-1} + \frac{1}{K} \right )^2 + c \beta, \\ \nonumber
&= %
\frac{1}{2} \frac{K}{K-1} \beta^2 - (1-c) \beta + \frac{1}{2} \frac{K-1}{K}
\end{align}
which is a ``smiling" parabola in $\beta$, with feasible minimum at $\beta^*=\mathrm{max}\{\frac{(1-c)(K-1)}{K}, 0\}$.
This means that if $c>1$ we get the (feasible) minimum at $(\alpha^*,\tilde{\alpha}^*,\beta^*)=(0,\mathrm{arccos}(\frac{-1}{K-1}),0)$, for which $\tilde{f}(\alpha^*,\tilde{\alpha}^*,\beta^*)=\frac{K-1}{2K}$. 
If $c\leq1$, we get minimum at $(\alpha^*,\tilde{\alpha}^*,\beta^*)=(0,\mathrm{arccos}(\frac{-1}{K-1}),\frac{(1-c)(K-1)}{K})$ with objective value of $\tilde{f}(\alpha^*,\tilde{\alpha}^*,\beta^*)= \frac{1}{2}\frac{K-1}{K}\left ( 1-(1-c)^2 \right ) = \frac{K-1}{K} \left ( c - \frac{1}{2}c^2 \right )$.

Assuming that $\beta \geq \frac{K-1}{K}$, the first term in \eqref{app2_Eq_lemma} is minimized (eliminated) by $\alpha^*=\mathrm{arccos}(\frac{K-1}{K\beta})$,
and the second term in \eqref{app2_Eq_lemma} is minimized (eliminated) by $\tilde{\alpha}^*=\mathrm{arccos}(\frac{-1}{K\beta})$.
Thus, we get
$
\tilde{f}(\alpha^*,\tilde{\alpha}^*,\beta) = c \beta,
$
which is minimized by $\beta^*=\frac{K-1}{K}$, and the objective value is $\tilde{f}(\alpha^*,\tilde{\alpha}^*,\beta^*)=c\frac{K-1}{K}$.
Since $c>0$, note that this value is always larger than the minimal value obtained for $\beta < \frac{K-1}{K}$.

To summarize, (i) if $c>1$ %
we get the minimizers $\tilde{f}(\alpha^*,\tilde{\alpha}^*,\beta^*=0)=\frac{K-1}{2K}$; (ii) 
If $c \leq 1$ %
we get the minimizer $\tilde{f}(\alpha^*=0,\tilde{\alpha}^*=\mathrm{arccos}(-\frac{1}{K-1}),\beta^*=\frac{(1-c)(K-1)}{K})=\frac{K-1}{K} \left ( c - \frac{1}{2}c^2 \right )$. %

\end{proof}

\newpage

\section{Proof of Theorem~\ref{thm_nc_of_deeper}}
\label{app:proof3}

\tomr{
We start by showing the zero within-class variability property of a minimizer $\H_1$, where we denote by $\{ \h_{k,i} \}$ the columns of $\H_1$: %
\begin{align}
\label{Eq_app3_poof_1_alt}
&\frac{1}{2N} \| \W_2 \W_1 \H_1 - \Y  \|_F^2 + \frac{\lambda_{W_2}}{2} \|\W_2\|_F^2 + \frac{\lambda_{W_1}}{2} \|\W_1\|_F^2 + \frac{\lambda_{H_1}}{2} \|\H_1\|_F^2 \\ \nonumber
    &= \frac{1}{2Kn} \sum_{k=1}^{K} n\frac{1}{n} \sum_{i=1}^{n} \| \W_2 \W_1 \h_{k,i} - \y_k \|_2^2 
+ \frac{\lambda_{W_2}}{2} \|\W_2\|_F^2 + \frac{\lambda_{W_1}}{2} \|\W_1\|_F^2 + \frac{\lambda_H}{2} \sum_{k=1}^{K} n\frac{1}{n} \sum_{i=1}^{n} \|\h_{k,i}\|_2^2 \\ \nonumber
    &\geq \frac{1}{2Kn} \sum_{k=1}^{K} n \left \| \W_2 \W_1 \frac{1}{n} \sum_{i=1}^{n} \h_{k,i} - \y_k \right \|^2 %
+ \frac{\lambda_{W_2}}{2} \|\W_2\|_F^2 + \frac{\lambda_{W_1}}{2} \|\W_1\|_F^2 + \frac{\lambda_H}{2} \sum_{k=1}^{K} n  \left \| \frac{1}{n} \sum_{i=1}^{n} \h_{k,i} \right \|_2^2,   
\end{align}
where 
we used Jensen's inequality, which (due to the strict convexity of %
$\|\cdot\|^2$) holds with equality iff $\h_{k,1}=\ldots=\h_{k,n}$ for all $k \in [K]$.
That is, the minimizer must exhibit zero within-class variability: $\H_1 = \overline{\H}_1 \otimes \1_n^\top$ for some $\overline{\H}_1 \in \mathbb{R}^{d \times K}$.}

\tomr{Next, we} 
are going to connect the minimization of the three-factors objective 
of \eqref{Eq_prob_deeper} \tomr{(after substituting $\H_1 = \overline{\H}_1 \otimes \1_n^\top$ and $\Y=\I_K \otimes \1_n^\top$)}
$$
f(\W_2,\W_1,\overline{\H}_1):=  \frac{1}{2K} \| \W_2 \W_1 \overline{\H}_1  - \I_K \|_F^2 
+ \frac{\lambda_{W_2}}{2} \|\W_2\|_F^2 + \frac{\lambda_{W_1}}{2} \|\W_1\|_F^2 + \frac{n\lambda_{H_1}}{2} \| \overline{\H}_1 \|_F^2
$$
with two sub-problems that include two-factors objectives. We will use the following lemma from \cite{zhu2021geometric} (which slightly generalizes a result from \cite{srebro2004learning}).
In this lemma, $\|\Z\|_*$ denotes the nuclear norm of the matrix $\Z$, i.e., the sum of its singular values.

\begin{lemma}[Lemma A.3 in \cite{zhu2021geometric}]
\label{app3_lemma_nuc}
For any fixed $\Z \in \mathbb{R}^{K \times N}$ and $\alpha>0$, we have 
\begin{align}
    \|\Z\|_* = \minim{\W,\H \,\, s.t. \,\, \W\H=\Z} \,\,\, \frac{1}{2} \left ( \frac{1}{\sqrt{\alpha}} \|\W\|_F^2 + \sqrt{\alpha} \|\H\|_F^2 \right ).
\end{align}
\tomt{Note that the minimizers $\W,\H$ obey $\W = \alpha^{1/4} \U \bSigma^{1/2} \R^\top$ and $\H = \alpha^{-1/4} \R \bSigma^{1/2} \V^\top$, where $\U\bSigma\V^\top$ is the SVD of $\Z$ and $\R$ is any orthonormal matrix of suitable dimensions.}
\end{lemma}

The first sub-problem is derived as follows:
\begin{align}
    \label{Eq_lemma2_aux_general2_0_}
    &\minim{\W_2,\W_1,\overline{\H}_1} \,\, \frac{1}{2K} \| \W_2 \W_1 \overline{\H}_1 - \I_K \|_F^2 + \frac{\lambda_{W_2}}{2} \|\W_2\|_F^2 + \frac{\lambda_{W_1}}{2} \|\W_1\|_F^2 +  \frac{n\lambda_{H_1}}{2} \|  \overline{\H}_1 \|_F^2 \\
    &= \minim{\W_2,\W_1,\overline{\H}_1,\overline{\H} \,\, s.t. \,\, \overline{\H}=\W_1\overline{\H}_1} \,\, \frac{1}{2K} \| \W_2 \overline{\H} - \I_K \|_F^2 + \frac{\lambda_{W_2}}{2} \|\W_2\|_F^2 + \frac{\lambda_{W_1}}{2} \|\W_1\|_F^2 +  \frac{n\lambda_{H_1}}{2} \| \overline{\H}_1 \|_F^2 \\ 
    &= \minim{\W_2,\W_1,\overline{\H}_1,\overline{\H} \,\, s.t. \,\, \overline{\H}=\W_1\overline{\H}_1} \,\, \frac{1}{2K} \| \W_2 \overline{\H} - \I_K \|_F^2 + \frac{\lambda_{W_2}}{2} \|\W_2\|_F^2 \\ \nonumber
    &\hspace{10mm} + \sqrt{\lambda_{W_1} n\lambda_{H_1}} \frac{1}{2} \left ( \frac{1}{\sqrt{n\lambda_{H_1}/\lambda_{W_1}}} \|\W_1\|_F^2 + \sqrt{n\lambda_{H_1}/\lambda_{W_1}} \| \overline{\H}_1 \|_F^2 \right ) \\  
    \label{Eq_lemma2_aux_general2_ineq_}
    &\tomr{=} \minim{\W_2,\overline{\H}} \,\, \frac{1}{2K} \| \W_2 \overline{\H} - \I_K \|_F^2 + \frac{\lambda_{W_2}}{2} \|\W_2\|_F^2 \\ \nonumber
    &\hspace{10mm} + \sqrt{\lambda_{W_1} n\lambda_{H_1}} \minim{\W_1,\overline{\H}_1 \,\, s.t. \,\, \W_1\overline{\H}_1=\overline{\H}} \,\, \frac{1}{2} \left ( \frac{1}{\sqrt{n\lambda_{H_1}/\lambda_{W_1}}} \|\W_1\|_F^2 + \sqrt{n\lambda_{H_1}/\lambda_{W_1}} \| \overline{\H}_1 \|_F^2 \right ) \\      \label{Eq_lemma2_aux_general2_nonconvex_}
    &= \minim{\W_2, \overline{\H}} \,\, f_1(\W_2,\overline{\H}) := \,\, \frac{1}{2K} \| \W_2 \overline{\H} - \I_K \|_F^2 + \frac{\lambda_{W_2}}{2} \|\W_2\|_F^2  + \sqrt{n\lambda_{W_1} \lambda_{H_1}} \| \overline{\H} \|_*
\end{align}
where the last equality follows from Lemma~\ref{app3_lemma_nuc}. %

With very similar steps, the second sub-problem is stated as:
\begin{align}
    \label{Eq_lemma2_aux_general3_0_}
    &\minim{\W_2,\W_1,\overline{\H}_1} \,\, \frac{1}{2K} \| \W_2 \W_1 \overline{\H}_1 - \I_K \|_F^2 + \frac{\lambda_{W_2}}{2} \|\W_2\|_F^2 + \frac{\lambda_{W_1}}{2} \|\W_1\|_F^2 +  \frac{n\lambda_{H_1}}{2} \| \overline{\H}_1 \|_F^2 \\
     \label{Eq_lemma2_aux_general3_nonconvex_}
    &\tomr{=} \minim{\W, \overline{\H}_1} \,\, f_2(\W,\overline{\H}_1) := \,\, \frac{1}{2K} \| \W \overline{\H}_1 - \I_K \|_F^2 + \frac{n\lambda_{H_1}}{2} \| \overline{\H}_1 \|_F^2  + \sqrt{\lambda_{W_2} \lambda_{W_1}} \| \W \|_*
\end{align}

Therefore, we can analyze the minimizers of \eqref{Eq_lemma2_aux_general2_nonconvex_} and \eqref{Eq_lemma2_aux_general3_nonconvex_} and translate the results to the minimizers of \eqref{Eq_prob_deeper}, using the characteristics of the minimizers in Lemma~\ref{app3_lemma_nuc}. 

Let us start with \eqref{Eq_lemma2_aux_general2_nonconvex_}: %
$$
f_1(\W_2,\overline{\H}) := \,\, \frac{1}{2K} \| \W_2 \overline{\H} - \I_K \|_F^2 + \frac{\lambda_{W_2}}{2} \|\W_2\|_F^2  + \sqrt{n \lambda_{W_1} \lambda_{H_1}} \| \overline{\H} \|_*.
$$

\tomr{
From $\frac{\partial f_1}{\partial \W_2}= \frac{1}{K} ( \W_2 \overline{\H} - \I_K) \overline{\H}^\top + \lambda_{W_2} \W_2 = \0$, observe that the minimizer w.r.t.~$\W_2$ is a closed-form function of $\overline{\H}$:
\begin{align}
\label{app3_Eq_W_func_H}
\W_2(\overline{\H}) = \overline{\H}^\top ( \overline{\H} \overline{\H}^\top + K \lambda_{W_2} \I_d )^{-1}.
\end{align}
Let us denote the compact SVD of $\overline{\H}=\U\S\V^\top$, where $\U \in \mathbb{R}^{d \times K}$ is a partial orthonormal matrix and $\V \in \mathbb{R}^{K \times K}$ is an orthonormal matrix and $\S \in \mathbb{R}^{K \times K}$ is a diagonal matrix with $\{ s_k \}_{k=1}^K$ on its diagonal. 
From \eqref{app3_Eq_W_func_H} %
we have that a minimizer $(\W_2,\overline{\H})$ obeys
\begin{align*}
    &\W_2 %
    = \V \S \U^\top (\U \S^2 \U^\top + K\lambda_{W_2} \I_d )^{-1} = \V \S (\S^2 + K\lambda_{W_2} \I_K )^{-1} \U^\top, \\
    &\W_2 \overline{\H} - \I_K = \V \left ( \S (\S^2 + K\lambda_{W_2}\I_d)^{-1} \S - \I_K \right ) \V^\top = \V \mathrm{diag} \left \{ \frac{-K\lambda_{W_2}}{(s_1^2 + K\lambda_{W_2})} , ... , \frac{-K\lambda_{W_2}}{(s_K^2 + K\lambda_{W_2})} \right \} \V^\top.
\end{align*}
Substituting these expressions in the objective, and using the fact that unitary operators do not change the Frobenius and nuclear norms, we can express the objective as a function of the singular values $\{s_k\}$
\begin{align}
\label{Eq_app3_poof_2_alt}
f_1(\W_2, \overline{\H}) &= \frac{1}{2K} \sum_{k=1}^K \frac{(K\lambda_{W_2})^2}{(s_k^2 + K\lambda_{W_2})^2} + \frac{\lambda_{W_2}}{2} \sum_{k=1}^K \frac{s_k^2}{(s_k^2 + K\lambda_{W})^{2}} + \sqrt{n\lambda_{W_1}\lambda_{H_1}} \sum_{k=1}^K s_k \\ \nonumber
&= \sum_{k=1}^K \left ( \frac{\lambda_{W_2}}{2} \frac{1}{s_k^2 + K\lambda_{W_2}} + \sqrt{n\lambda_{W_1}\lambda_{H_1}} s_k \right ).
\end{align}
At this point, we already see that the objective is separable w.r.t.~the singular values, which implies that the minimizer obeys $s_1 = ... = s_K =:s$.
Therefore, $\overline{\H} = s \U \V^\top$ and $\W_2 = \frac{s}{s^2+K\lambda_{W_2}} \V \U^\top \propto \overline{\H}^\top$.}

\tomr{
The flat spectrum of $\overline{\H}$ implies that $\overline{\H}^\top \overline{\H} \propto \I_K$, since 
$
\overline{\H}^\top \overline{\H} = s^2 \V \U^\top \U \V^\top = s^2 \V \V^\top = s^2 \I_K.
$
Similarly, $\W_2\overline{\H} \propto \I_K$ %
and $\W_2\W_2^\top \propto \I_K$. 
Denoting $\sigma_{\overline{H}}:=s$ and $\sigma_{W}:=\frac{s}{s^2+K\lambda_{W_2}}$, 
the above results imply that %
a minimizer %
$(\W_2,\overline{\H})$ is given by
\begin{align}
\label{Eq_app3_poof_3_alt}
    \W_2 &= \sigma_{W} \R^\top \in \mathbb{R}^{K \times d} \\
\label{Eq_app3_poof_4_alt}    
    \overline{\H} &= \sigma_{\overline{H}} \R \in \mathbb{R}^{d \times K}
\end{align}
with arbitrary orthonormal matrix $\R \in \mathbb{R}^{d \times K}$ ($\R^\top\R = \I_K$).}

\tomr{
Potentially, 
the values of $\sigma_W$ and $\sigma_{\overline{H}}$ can be determined by minimizing the simplified objective (obtained by substituting \eqref{Eq_app3_poof_3_alt}-\eqref{Eq_app3_poof_4_alt}:} %

\begin{align}
\label{app3_obj_W_Hbar3}
f_1(\W_2,\overline{\H}) &= \frac{1}{2} ( \sigma_{{W}} \sigma_{\overline{H}} - 1 )^2 + K \frac{\lambda_{W_2}}{2} \sigma_W^2 + K \sqrt{n \lambda_{W_1} \lambda_{H_1}} \sigma_{\overline{H}}.
\end{align}
The derivatives are given by
\begin{align}
\label{app3_obj_W_Hbar4}
\frac{\partial}{\partial \sigma_W} f_1 &= \sigma_{\overline{H}} ( \sigma_{{W}} \sigma_{\overline{H}} - 1 ) + K \lambda_{W_2} \sigma_W = 0, \\
\label{app3_obj_W_Hbar4b}
\frac{\partial}{\partial \sigma_{\overline{H}}} f_1 &= \sigma_{W} ( \sigma_{{W}} \sigma_{\overline{H}} - 1 ) + K \sqrt{n \lambda_{W_1} \lambda_{H_1}} = 0,
\end{align}
implying that $\lambda_{W_2} \sigma_W^2 = \sqrt{n \lambda_{W_1} \lambda_{H_1}} \sigma_{\overline{H}}$. Plugging $\sigma_{\overline{H}} = \frac{\lambda_{W_2}\sigma_W^2}{\sqrt{n \lambda_{W_1} \lambda_{H_1}}}$ in \eqref{app3_obj_W_Hbar4b} we get
$$
\lambda_{W_2} \sigma_{{W}}^4 - \sqrt{n \lambda_{W_1} \lambda_{H_1}} \sigma_{W} + K n \lambda_{W_1} \lambda_{H_1} = 0
$$
The value of $\sigma_{W}$ can be computed numerically as the positive root of the above 4th degree polynomial (the analytical result is extremely cumbersome) and the same goes for the value of $\sigma_{\overline{H}}$.
\tomr{Note that an attempt to determine $s$ by minimization of \eqref{Eq_app3_poof_2_alt} also leads to a challenging 4th degree polynomial.} 
Yet, even without stating these exact constants we can summarize our findings for \eqref{Eq_lemma2_aux_general2_nonconvex_} as follows. We %
have shown that the minimizers obey %
$\overline{\H}=\sigma_{\overline{H}}\R$ and $\W_2 = \sigma_{W}\R^\top$ for some non-negative constants $\sigma_{\overline{H}}, \sigma_{W}$ (which depend on $K,n,\lambda_{W_2},\lambda_{W_1},\lambda_{H_1}$) and any orthonormal matrix $\R \in \mathbb{R}^{d \times K}$.
Therefore, $\W_2 \propto \overline{\H}^\top$, and 
$$
\W_2 \overline{\H} \propto \overline{\H}^\top \overline{\H} \propto  \W_2 \W_2^\top \propto  \I_K.
$$

From 
\tomr{$\H_1 = \overline{\H}_1 \otimes \1_n^\top$ and
Lemma~\ref{app3_lemma_nuc} (which factorizes $\overline{\H}=\sigma_{\overline{H}}\R$ to $\W_1 \overline{\H}_1$)} we know that the minimal objective value of \eqref{Eq_lemma2_aux_general2_nonconvex_} is attained by the minimizers $\W_1, \H_1$ of \eqref{Eq_prob_deeper} 
for which 
we have %
$\W_1 = \sqrt[4]{n\lambda_{H_1}/\lambda_{W_1}} \sqrt{\sigma_{\overline{H}}} \R \tilde{\R}^\top$ 
and 
$\H_1 = \frac{1}{\sqrt[4]{n\lambda_{H_1}/\lambda_{W_1}}} \sqrt{\sigma_{\overline{H}}} \tilde{\R} \otimes \1_n^\top$ 
for any orthonormal matrix $\tilde{\R} \in \mathbb{R}^{d \times K}$. %

We conclude that for \tomr{$d \geq K$} and $(\W_2^*,\W_1^*,\H_1^*)$ being a (nonzero) global minimizer of \eqref{Eq_prob_deeper}, we have that $\W_1^*\H_1^*$ collapses to an orthogonal $d \times K$ frame, and $\W_2^{*\top}$ is an orthogonal %
$d \times K$
matrix that is aligned with $\W_1^*\H_1^*$.

Analyzing the minimizers of \eqref{Eq_lemma2_aux_general3_nonconvex_} by steps which are  very similar to those used for \eqref{Eq_lemma2_aux_general2_nonconvex_} 
\tomr{(essentially, using $\frac{\partial f_2(\W,\overline{\H}_1)}{\partial \overline{\H}_1} = \0$ to express the minimizer w.r.t.~$\overline{\H}_1$ as function of $\W$, and expressing the objective using the singular values of $\W$)} %
yields the following.

The minimizers of \eqref{Eq_lemma2_aux_general3_nonconvex_} obey $\H_1 = \overline{\H}_1 \otimes \1_n^\top$, where $\overline{\H}_1=\sigma_{\overline{H}_1}\tilde{\R}$ and $\W = \sigma_{W}\tilde{\R}^\top$ for some non-negative constants $\sigma_{\overline{H}_1}, \sigma_{W}$ (which depend on $K,n,\lambda_{W_2},\lambda_{W_1},\lambda_{H_1}$) and any orthonormal matrix $\tilde{\R} \in \mathbb{R}^{d \times K}$.
Therefore, $\W \propto \overline{\H}_1^\top$, and 
$$
\W \overline{\H}_1 \propto \overline{\H}_1^\top \overline{\H}_1 \propto  \W \W^\top \propto  \I_K.
$$
Now, since $\W = \sigma_{W} \tilde{\R}^\top$, from Lemma~\ref{app3_lemma_nuc} we know that the minimal objective value of \eqref{Eq_lemma2_aux_general3_nonconvex_} is attained by the minimizers $\W_2, \W_1$ of \eqref{Eq_prob_deeper} 
for which 
we have %
$\W_2 = \sqrt[4]{\lambda_{W_1}/\lambda_{W_2}} \sqrt{\sigma_{W}} \R^\top$ 
and 
$\W_1 = \frac{1}{\sqrt[4]{\lambda_{W_1}/\lambda_{W_2}}} \sqrt{\sigma_{W}} \R \tilde{\R}^\top$ 
for any orthonormal matrix $\R \in \mathbb{R}^{d \times K}$.

We conclude that for \tomr{$d \geq K$} and $(\W_2^*,\W_1^*,\H_1^*)$ being a (nonzero) global minimizer of \eqref{Eq_prob_deeper}, we have that $\H_1^*$ collapses to an orthogonal $d \times K$ frame, and $(\W_2^*\W_1^*)^\top$ is an orthogonal %
$d \times K$
matrix that is aligned with $\H_1^*$.

\newpage

\section{On the Within-Class Variability Metric NC1}
\label{app:lemma_nc1}

In this section, we discuss some properties of the within-class variability of the features $\H_1$ and $\H_2:=\W_1\H_1$ for the model in \eqref{Eq_prob_deeper}.
First, let us define the metric $NC_1$ that is used to measure the within-class variability. Note that this metric is related to the classical Fisher's ratio.
For a given (organized) features matrix $\H = \left [ \h_{1,1}, \ldots, \h_{1,n}, \h_{2,1}, \ldots , \h_{K,n}  \right ] \in \mathbb{R}^{d \times Kn}$, denote the per-class and global means as
$\overline{\h}_k := \frac{1}{n}\sum_{i=1}^n \h_{k,i}$ and $\overline{\h}_G := \frac{1}{Kn}\sum_{k=1}^k \sum_{i=1}^n \h_{k,i}$, respectively. Define the within-class and between-class $d \times d$ covariance matrices
$$
\bSigma_W(\H) := \frac{1}{Kn}\sum_{k=1}^K \sum_{i=1}^n (\h_{k,i}-\overline{\h}_k)(\h_{k,i}-\overline{\h}_k)^\top,
$$
$$
\bSigma_B(\H):=\frac{1}{K}\sum_{k=1}^K (\overline{\h}_k-\overline{\h}_G)(\overline{\h}_k-\overline{\h}_G)^\top.
$$
We define the corresponding within-class variability metric as
\begin{align}
\label{app_Eq_NC1}
NC_1(\H) := \frac{1}{K}\tr \left ( \bSigma_W(\H) \bSigma_B^\dagger(\H) \right ),
\end{align}
where $\bSigma_B^\dagger$ denotes the pseudoinverse of $\bSigma_B$.

From the definitions above, observe that $\bSigma_W(\H_2) = \W_1 \bSigma_W(\H_1) \W_1^\top$ and $\bSigma_B(\H_2) = \W_1 \bSigma_B(\H_1) \W_1^\top$.
Therefore,
\tomr{assuming that $( \W_1 \bSigma_B(\H_1) \W_1^\top )^\dagger \approx  \W_1^{\top\dagger}  \bSigma_B^\dagger(\H_1) \W_1^{\dagger}$, we have that}
\begin{align}
\label{app_Eq_NC1_2}
NC_1(\H_2) &= \frac{1}{K}\tr \left ( \W_1 \bSigma_W(\H_1) \W_1^\top ( \W_1 \bSigma_B(\H_1) \W_1^\top )^\dagger \right ) \\ \nonumber
&\approx \frac{1}{K}\tr \left ( \W_1 \bSigma_W(\H_1) \W_1^\top \W_1^{\top\dagger}  \bSigma_B^\dagger(\H_1) \W_1^{\dagger} \right ) \\ \nonumber
&= \frac{1}{K}\tr \left ( \W_1^{\dagger} \W_1 \bSigma_W(\H_1) \left ( \W_1^{\dagger} \W_1 \right )^\top \bSigma_B^\dagger(\H_1) \right ).
\end{align}
Now, by their definitions, the columns of $\bSigma_W(\H_1)$ and $\bSigma_B(\H_1)$ are in the range of $\H_1$.
Thus, since $\W_1^{\dagger} \W_1$ is an orthogonal projection matrix (onto the subspace spanned by the rows of $\W_1$), we have that 
$$
NC_1(\H_2) \approx \frac{1}{K}\tr \left ( \W_1^{\dagger} \W_1 \bSigma_W(\H_1) \left ( \W_1^{\dagger} \W_1 \right )^\top \bSigma_B^\dagger(\H_1) \right )
= \frac{1}{K}\tr \left ( \bSigma_W(\H_1) \bSigma_B^\dagger(\H_1) \right ) = NC_1(\H_1)
$$
is guaranteed when there are no columns of $\H_1$ in the null space of $\W_1$.
One such case is at initialization, when $\W_1$ is initialized by continuous random distribution and thus its rows span $\mathbb{R}^d$ with probability 1.
Moreover, after random initialization, we empirically observed that $\H_1$ and $\H_2$ also have similar $NC_1$ along gradient-based optimization (see Figure~\ref{fig:thm3}), which is due to having similar $K$ dimensional subspaces dominantly spanned by the columns of $\H_1$ and the rows of $\W_1$ (as well as those of $\W_2$).
At convergence to the a global minimizer, again it is guaranteed that there are no columns of $\H_1$ in the null space of $\W_1$.
Specifically, %
as demonstrated in the proof of Theorem~\ref{thm_nc_of_deeper},
the global minimizers necessarily have that $\W_2^{*\top}, \W_1^{*\top}$ and $\H_1^*$ have exactly the same $K$ dimensional range (column space).
Briefly, denoting the objective of \eqref{Eq_prob_deeper} by $f$, 
this follows from $\W_2^*\W_2^{*\top} \propto \I_K$, as well as  
$\lambda_{W_2} \W_2^{*\top}\W_2^* = \lambda_{W_1}\W_1^* \W_1^{*\top}$ and $\lambda_{W_1} \W_1^{*\top}\W_1^* = \lambda_{H_1} \H_1^* \H_1^{*\top}$,
where the last two equalities follow from
$\W_1^\top \frac{\partial f}{\partial \W_1} - \frac{\partial f}{\partial \H_1} \H_1^\top = \0$ and $\W_2^\top \frac{\partial f}{\partial \W_2} - \frac{\partial f}{\partial \W_1} \W_1^\top = \0$, respectively.

\newpage

\section{Proof of Theorem~\ref{thm_nc_of_deeper_nonlin}}
\label{app:proof4}

Let $(\W_2^*, \W_1^*, \H_1^*)$ be a minimizer of \eqref{Eq_prob_nonlin}. 
Similar to the arguments in the proof of Theorem~\ref{thm_nc_of_deeper}, the construction $(\W_2,\H) = (\W_2^*, \W_1^*\H_1^* )$ is also a minimizer of the following sub-problem
\begin{align}
\label{app_Eq_subprob_nonlin}
    \tilde{f}_1(\W_2,\H) := \,\, \frac{1}{2Kn} \| \W_2 \sigma(\H) - \Y \|_F^2 + \frac{\lambda_{W_2}}{2} \|\W_2\|_F^2  + \sqrt{\lambda_{W_1} \lambda_{H_1}} \| \H \|_*.
\end{align}
\tomr{The goal is} 
to show that the construction $(\W_2,\H) = (\W_2^*, \sigma(\W_1^*\H_1^*) )$ is also a minimizer of 
\begin{align}
\label{app_Eq_subprob_from_lin_model}
    f_1(\W_2,\H) := \,\, \frac{1}{2Kn} \| \W_2 \H - \Y \|_F^2 + \frac{\lambda_{W_2}}{2} \|\W_2\|_F^2  + \sqrt{\lambda_{W_1} \lambda_{H_1}} \| \H \|_*.
\end{align}
Recall that, following the proof of Theorem~\ref{thm_nc_of_deeper}, for a given minimizer of \eqref{Eq_prob_deeper}, $(\tilde{\W}_2^*,\tilde{\W}_1^*,\tilde{\H}_1^*)$, we have that  $(\W_2,\H)=(\tilde{\W}_2^*,\tilde{\W}_1^*\tilde{\H}_1^*)$
minimizes \eqref{app_Eq_subprob_from_lin_model}.
Therefore, this will imply that the orthogonal collapse and alignment properties of $\tilde{\W}_2^*$ and $\tilde{\W}_1^*\tilde{\H}_1^*$, 
which have been obtained by analyzing $f_1(\W_2,\H)$ in Appendix~\ref{app:proof3}, 
carry on to $\W_2^*$ and $\sigma(\W_1^*\H_1^*)$ constructed from global minimizers of \eqref{Eq_prob_nonlin}.

We begin by considering an intermediate objective:
\begin{align}
\label{app_Eq_subprob_hybrid}
    \overline{f}_1(\W_2,\H) := \,\, \frac{1}{2Kn} \| \W_2 \sigma(\H) - \Y \|_F^2 + \frac{\lambda_{W_2}}{2} \|\W_2\|_F^2  + \sqrt{\lambda_{W_1} \lambda_{H_1}} \| \sigma(\H) \|_*.
\end{align}
Clearly, $\minim{\W_2, \H} \,\, f_1(\W_2,\H) \leq \minim{\W_2, \H} \,\, \overline{f}_1(\W_2,\H)$, 
since the ReLUs in \eqref{app_Eq_subprob_hybrid} can be translated to a non-negativity constraint on $\H$ that reduces the feasible set of the minimization problem.
Yet, let us show that this inequality is not strict.

Essentially, showing that $\minim{\W_2, \H} \,\, f_1(\W_2,\H) = \minim{\W_2, \H} \,\, \overline{f}_1(\W_2,\H)$ is translated to proving that \eqref{app_Eq_subprob_from_lin_model} has a non-negative minimizer.
To this end, we use the following properties of the minimizers that has been shown in the proof of Theorem~\ref{thm_nc_of_deeper}:
$\tilde{\H}^*$ has the structure $\tilde{\H}^* = \overline{\H} \otimes \1_n^\top$ and
\begin{align}
    \tilde{\W}_2^* &= %
    \sigma_W^*
    \R^\top \in \mathbb{R}^{K \times d} \\
    \overline{\H} &= %
    \sigma_{\overline{H}}^* \R
    \in \mathbb{R}^{d \times K}
\end{align}
where %
\tomr{$\sigma_W^*$ and $\sigma_{\overline{H}}^*$ are non-negative scalars (singular values)}
and $\R \in \mathbb{R}^{d \times K}$ can be any orthonormal matrix ($\R^\top\R = \I_K$). (The freedom in $\R$ is due to the fact that the problem can be expressed only in terms of the singular values). 
Now, we can get the existence of the desired non-negative matrices
by considering
$$
\R = \begin{bmatrix}
   \I_K \\
   \0_{(d-K) \times K} 
   \end{bmatrix},
$$
for which
$$
\tilde{\W}_2^* = \sigma_W^* %
\begin{bmatrix}
   \I_K &
   \0_{K \times (d-K)} 
   \end{bmatrix}
$$
$$
\tilde{\W}_1^*\tilde{\H}_1^* = \sigma_{\overline{H}}^*
\begin{bmatrix}
   \I_K \\
   \0_{(d-K) \times K} 
   \end{bmatrix} 
   \otimes \1_n^\top
$$
are clearly non-negative. 

The above result also implies that the set of minimizers of \eqref{app_Eq_subprob_hybrid} is a subset of the set of minimizers of \eqref{app_Eq_subprob_from_lin_model}.
Thus, such minimizers carry the property that $\sigma(\H^*)$ is a collapsed orthogonal matrix: $\sigma(\H^*) = \overline{\H} \otimes \1_n^\top$ for some non-negative $\overline{\H} \in \mathbb{R}^{d \times K}$ that obeys $\overline{\H}_2^\top \overline{\H}_2 = \alpha \I_k$ for some positive scalar $\alpha$.

Now, let us compare $\overline{f}_1(\W_2,\H)$ with $\tilde{f}_1(\W_2,\H)$. 
Observe that, 
intuitively,  
$\minim{\W_2, \H} \,\, \tilde{f}_1(\W_2,\H) \leq \minim{\W_2, \H} \,\, \overline{f}_1(\W_2,\H)$
because, apparently, %
negative entries of $\H$ can be used to reduce $ \| \H \|_*$. %
Formally, %
we have that
\begin{align*}
    \minim{\W_2, \H} \,\, \tilde{f}_1(\W_2,\H)  &= \minim{\W_2, \tilde{\H}, \H: \sigma(\H)=\sigma(\tilde{\H})} \,\, \overline{f}_1(\W_2,\tilde{\H}) + \sqrt{\lambda_{W_1} \lambda_{H_1}} ( \| \H \|_* - \| \sigma(\tilde{\H}) \|_* ) \\ \nonumber
    &= \minim{\W_2, \tilde{\H}} \,\, \overline{f}_1(\W_2,\tilde{\H}) + \minim{\H: \sigma(\H)=\sigma(\tilde{\H})} \,\, \sqrt{\lambda_{W_1} \lambda_{H_1}} ( \| \H \|_* - \| \sigma(\tilde{\H}) \|_* )
\end{align*}
where the inner minimization is non-positive (observe that it equals zero for the feasible point $\H = \sigma(\tilde{\H})$).

\tomr{
However, in any numerical experiment that we performed (much beyond those that are presented in this paper) we observed that a minimizer of $\tilde{f}_1(\W_2,\H)$, namely $(\W_2^*,\H^*)=(\W_2^*, \W_1^*\H_1^*)$ constructed from a minimizer of \eqref{Eq_prob_nonlin}, obeys that $\|\H^*\|_* = \|\sigma(\H^*)\|_*$.
This implies that negative entries of $\H$ in $\tilde{f}_1(\W_2,\H)$ do not contribute to reducing the objective. 
Therefore, $\minim{\W_2, \H} \,\, \tilde{f}_1(\W_2,\H) = \minim{\W_2, \H} \,\, \overline{f}_1(\W_2,\H)$, 
and $(\W_2,\H)=(\W_2^*, \sigma(\H^*))$ is also a minimizer of $\overline{f}_1(\W_2,\H)$.}

\tomr{
Currently, the property $\|\H^*\|_* = \|\sigma(\H^*)\|_*$ is considered as an assumption of the theorem. 
Yet, we conjecture that it necessarily holds, such that a version of the theorem without this assumption may be published in the future.} 

From the above equivalences between $\tilde{f}_1$ and ${f}_1$ through $\overline{f}_1$, we obtain the desired result that if $(\W_2^*,\H^*)$ is a (global) minimizer of $\tilde{f}_1(\W_2,\H)$, then $(\W_2^*,\sigma(\H^*))$ is a (global) minimizer of $f_1(\W_2,\H)$ \tomr{and thus carries the orthogonal collapse properties}.

\newpage

\section{Proof of Theorem~\ref{thm_nc_of_asymp}}
\label{app:proof_asymp}

As stated in the theorem, we consider \eqref{Eq_prob} with $\lambda_H = \frac{\tilde{\lambda}_H}{n}$: 
\begin{align}
\label{app_Eq_prob_modified}
    \minim{\W,\H} \,\, %
\frac{1}{2Kn} \| \W \H - \Y \|_F^2 + \frac{\lambda_W}{2} \|\W\|_F^2 + \frac{\tilde{\lambda}_H}{2n} \|\H\|_F^2, %
\end{align}
and denote by $(\W^*,\H^*)$ a global minimizer.
From Theorem~\ref{thm_nc_of} we have that $\H^*=\overline{\H}\otimes \1_n^\top$ and $\W^* = \sqrt{\tilde{\lambda}_H/\lambda_W} \overline{\H}^\top$ for some $\overline{\H}\in \mathbb{R}^{d \times K}$ that obeys $\overline{\H}^\top \overline{\H} = \rho \I_K = (1-K\sqrt{\tilde{\lambda}_H\lambda_W})\sqrt{\frac{\lambda_W}{\tilde{\lambda}_H}} \I_K = (\sqrt{\frac{\lambda_W}{\tilde{\lambda}_H}} - K \lambda_W) \I_K$.

Note that for any value of $n$, we have that 
$(\W,\H) = (\W^*,\H^*_n := \overline{\H}\otimes \1_n^\top)$ is a global minimizer of \eqref{app_Eq_prob_modified}.  

We turn to examine \eqref{app_Eq_prob_modified} for fixed %
$\H$ 
and minimization only w.r.t.~$\W$. Namely,
\begin{align}
\label{app_Eq_fixed_H}
\hat{\W}_n = \argmin{\W} \,\, %
\frac{1}{2Kn} \| \W {\H} - \Y \|_F^2 + \frac{\lambda_W}{2} \|\W\|_F^2.
\end{align}
This strongly convex problem has the following closed-form solution
\begin{align}
\label{app_Eq_fixed_H_sol}
\hat{\W}_n(\H) = \frac{1}{Kn} \Y \H^\top \left ( \frac{1}{Kn} \H \H^\top + \lambda_W \I_d \right )^{-1}.
\end{align} 

Recalling that $\Y=\I_K \otimes \1_n^\top$,
for $\H=\H^*_n=\overline{\H}\otimes \1_n^\top$ we have that
\begin{align}
\label{app_Eq_fixed_H_sol1}
\hat{\W}_n(\H^*_n) &= \frac{1}{Kn} (\I_K \overline{\H}^\top \otimes \1_n^\top \1_n) \left ( \frac{1}{Kn} (\overline{\H} \overline{\H}^\top \otimes \1_n^\top \1_n) + \lambda_W \I_d \right )^{-1} \\ \nonumber
& = \frac{1}{K} \overline{\H}^\top \left ( \frac{1}{K} \overline{\H} \overline{\H}^\top + \lambda_W \I_d \right )^{-1}.
\end{align} 
This expression can be simplified as follows
\begin{align}
\label{app_Eq_fixed_H_sol2}
\hat{\W}_n(\H^*_n) %
&= \frac{1}{K\lambda_W} \overline{\H}^\top \left ( \frac{1}{K\lambda_W} \overline{\H} \overline{\H}^\top + \I_d \right )^{-1} \\ \nonumber
&= \frac{1}{K\lambda_W} \left ( \frac{\rho}{K\lambda_W}\I_K  + \I_K \right )^{-1}\overline{\H}^\top \\ \nonumber
&= \frac{1}{K\lambda_W + \rho} \overline{\H}^\top, 
\end{align} 
where the second equality follows from the ``push-through identity" and the fact that $\overline{\H}^\top \overline{\H} = \rho \I_K$. 
Note that, as expected, if we fixed $\H$ to be $\H^*_n$, a global minimizer of the joint optimization w.r.t.~$(\W,\H)$, then we get $\hat{\W}_n=\W^*$. Indeed, $\hat{\W}_n(\H^*_n) =  \frac{1}{K\lambda_W + \rho} \overline{\H}^\top = \frac{1}{\sqrt{\lambda_W/\tilde{\lambda}_H}} \overline{\H}^\top = \W^*$.

Let us turn to examine $\hat{\W}_n$ for $\H=\tilde{\H}_n$ where 
$\tilde{\H}_n := \overline{\H}\otimes \1_n^\top + \E_n$
with $\E_n \in \mathbb{R}^{d \times Kn}$ whose entries are i.i.d.~random variables with zero mean, variance $\sigma_e^2$, and finite fourth moment. 
Hence, $\Expp{\E_n}=\0$ and $\Expp{\E_n\E_n^\top}=Kn\sigma_e^2\I_d$.

Substituting $\H=\tilde{\H}_n$ in \eqref{app_Eq_fixed_H_sol}, we get
\begin{align}
\label{app_Eq_fixed_H_sol3}
\hat{\W}_n(\tilde{\H}_n) = \frac{1}{Kn} \Y \tilde{\H}_n^\top \left ( \frac{1}{Kn} \tilde{\H}_n \tilde{\H}_n^\top + \lambda_W \I_d \right )^{-1}.
\end{align} 
Based on the law of large numbers, as well as the convergence of sample covariance matrices of random variables with finite fourth moment \cite{vershynin2012close}, we have the following limits
\begin{align}
\label{app_Eq_fixed_H_limits}
&\frac{1}{Kn} \Y \tilde{\H}_n^\top = \frac{1}{K}\overline{\H} + \frac{1}{Kn} (\I_K \otimes \1_n^\top) \E_n^\top \xrightarrow[n \xrightarrow{} \infty]{a.s.} \frac{1}{K}\overline{\H}, \\ \nonumber
&\frac{1}{Kn} \tilde{\H}_n \tilde{\H}_n^\top = \frac{1}{K}\overline{\H}\overline{\H}^\top + \frac{1}{Kn} (\overline{\H} \otimes \1_n^\top) \E_n^\top + \frac{1}{Kn} \E_n (\overline{\H}^\top \otimes \1_n) + \frac{1}{Kn} \E_n \E_n^\top
\xrightarrow[n \xrightarrow{} \infty]{a.s.} \frac{1}{K}\overline{\H}\overline{\H}^\top + \sigma_e^2 \I_d.
\end{align} 
Therefore, 
\begin{align}
\label{app_Eq_fixed_H_sol4}
\hat{\W}_n(\tilde{\H}_n) \xrightarrow[n \xrightarrow{} \infty]{a.s.} \frac{1}{K}\overline{\H} \left ( \frac{1}{K}\overline{\H}\overline{\H}^\top + \sigma_e^2 \I_d + \lambda_W \I_d \right )^{-1}.
\end{align} 
Repeating the simplifications of \eqref{app_Eq_fixed_H_sol2} (with $\sigma_e^2 + \lambda_W$ in lieu of $\lambda_W$) we get
\begin{align}
\label{app_Eq_fixed_H_sol5}
\hat{\W}_n(\tilde{\H}_n) \xrightarrow[n \xrightarrow{} \infty]{a.s.} \frac{1}{K(\sigma_e^2 + \lambda_W) + \rho} \overline{\H}^\top
= \frac{1}{K\sigma_e^2 + \sqrt{\lambda_W/\tilde{\lambda}_H} } \overline{\H}^\top.
\end{align} 
Comparing \eqref{app_Eq_fixed_H_sol5} with $\W^* = \frac{1}{\sqrt{\lambda_W/\tilde{\lambda}_H}} \overline{\H}^\top$, we get the result that is stated in the theorem:
$$
\hat{\W}_n(\tilde{\H}_n) \xrightarrow[n \xrightarrow{} \infty]{a.s.}
\frac{\sqrt{\lambda_W/\tilde{\lambda}_H}}{K\sigma_e^2 + \sqrt{\lambda_W/\tilde{\lambda}_H}} \W^*
= \frac{1}{1+\sigma_e^2 K \sqrt{\tilde{\lambda}_H/\lambda_W}} \W^*.
$$

\subsection{Intuitive explanation of the result}
\label{app:proof_asymp_intuitive}

The intuition that the asymptotic consequence of $\E_n$, %
i.e., the deviation from ``perfectly" collapsed features, 
 will only be some attenuation of $\W^*$ can also be seen from expending the quadratic term in \eqref{app_Eq_fixed_H} for $\H = \overline{\H}\otimes \1_n^\top + \E_n$ and eliminating the terms that are linear in the zero-mean $\E_n$. 
Specifically, observe that
\begin{align}
\label{app_Eq_fixed_H_w_E}
&\frac{1}{2Kn} \| \W (\overline{\H}\otimes \1_n^\top + \E_n) - \Y \|_F^2 + \frac{\lambda_W}{2} \|\W\|_F^2 =
\frac{1}{2Kn} \| (\W \overline{\H}\otimes \1_n^\top - \Y) + \W \E_n \|_F^2 + \frac{\lambda_W}{2} \|\W\|_F^2 \\ \nonumber
&= \frac{1}{2Kn} \| \W \overline{\H}\otimes \1_n^\top - \Y \|_F^2
+ \frac{1}{Kn} \tr ( \E_n^\top \W^\top (\W \overline{\H}\otimes \1_n^\top - \Y) )
+ \frac{1}{2Kn} \| \W \E_n \|_F^2 + \frac{\lambda_W}{2} \|\W\|_F^2.
\end{align}
Now, suppose we take the limit $n \xrightarrow{} \infty$ only in the terms that include $\E_n$, we would get
\begin{align}
\label{app_Eq_fixed_H_limits2}
&\frac{1}{Kn} \tr ( \E_n^\top \W^\top (\W \overline{\H}\otimes \1_n^\top - \Y) ) \xrightarrow[n \xrightarrow{} \infty]{a.s.} \0, \\ \nonumber
&\frac{1}{2Kn} \| \W \E_n \|_F^2 = \frac{1}{2Kn} \tr ( \E_n \E_n^\top \W^\top \W )
\xrightarrow[n \xrightarrow{} \infty]{a.s.}  \frac{1}{2} \sigma_e^2 \tr (\W^\top \W),
\end{align} 
under which \eqref{app_Eq_fixed_H_w_E} can be interpreted as
\begin{align}
\label{app_Eq_fixed_H_w_E2}
\frac{1}{2Kn} \| \W \overline{\H}\otimes \1_n^\top - \Y \|_F^2
+ \frac{\sigma_e^2}{2} \| \W \|_F^2 + \frac{\lambda_W}{2} \|\W\|_F^2.
\end{align}
This hints that, asymptotically, the minimizer $\hat{\W}$ would be similar to the minimizer that is obtained for the case of $\sigma_e=0$ (as shown above, this is in fact $\W^*$) up to some scaling.

The above intuition is aligned with the results of Theorem~\ref{thm_nc_of_asymp}.
Yet, contrary to the proof of the theorem, it does not require having a closed-form expression for the minimizer $\hat{\W}$.
Interestingly, this allows us to generalize it to the extended UFMs.
Specifically, consider the model in \eqref{Eq_prob_deeper} with fixed $\H_1 = \overline{\H}_1 \otimes \1_n^\top + \E_n$, where $(\W_2^*,\W_1^*,\H_1^*=\overline{\H}_1 \otimes \1_n^\top)$ is a global minimizer (as stated in Theorem~\ref{thm_nc_of_deeper}).
Namely,
\begin{align}
\frac{1}{2Kn} \| \W_2 \W_1 (\overline{\H}_1 \otimes \1_n^\top + \E_n) - \Y \|_F^2 + \frac{\lambda_{W_2}}{2} \|\W_2\|_F^2 + \frac{\lambda_{W_1}}{2} \|\W_1\|_F^2
\end{align}
Repeating the above heuristic, asymptotically, we may interpret this objective as 
\begin{align}
\frac{1}{2Kn} \| \W_2 \W_1 \overline{\H}_1 \otimes \1_n^\top - \Y \|_F^2 
+ \frac{\sigma_e^2}{2} \| \W_2 \W_1 \|_F^2
+ \frac{\lambda_{W_2}}{2} \|\W_2\|_F^2 + \frac{\lambda_{W_1}}{2} \|\W_1\|_F^2,
\end{align}
which maintains many of the properties of the model analyzed in Theorem~\ref{thm_nc_of_deeper}, such as invariance to various orthogonal transformations and the ability to restate the problem as optimization on the singular values of $\W_2, \W_1$ and $\H_1$ (as done in the proof in Appendix~\ref{app:proof3}). 
Again, this hints that, asymptotically, the minimizer $(\hat{\W_2}, \hat{\W_1})$ would be similar to the minimizer that is obtained for the case without $\E_n$, up to some scaling.
While we defer a rigorous study of the effect of fixed features matrix $\H_1$ on the extended UFMs for future research, the discussion here demonstrates the feasibility of this goal.

\newpage

\section{More Numerical Results for the Unconstrained Features Model}
\label{app:exp}

In this section, we present more numerical results, for experiments that are similar to those in Section~\ref{sec:exps} but with different configurations. The definitions of the NC metrics appear in Section~\ref{sec:exps}.

Figure~\ref{fig:app_thm1} corroborates Theorem~\ref{thm_nc_of} for $K=5, d=20, n=100$, %
$\lambda_W=0.005$ and $\lambda_H=0.001$ (no bias is used, equivalently $\lambda_b \xrightarrow{} \infty$). Both $\W$ and $\H$ are initialized with standard normal distribution and are optimized with plain gradient descent with step-size 0.1.

Figure~\ref{fig:app_thm2} corroborates Theorem~\ref{thm_nc_simplex_etf} for $K=5, d=20, n=100$, %
$\lambda_W=0.005$ and $\lambda_H=0.001$ and $\lambda_b=0$. All $\W$, $\H$ and $\b$ are initialized with standard normal distribution and are optimized with plain gradient descent with step-size 0.1.

Figure~\ref{fig:app_thm3} corroborates %
Theorem~\ref{thm_nc_of_deeper}
for $K=5, d=20, n=100$, %
$\lambda_{W_2}=0.005$, $\lambda_{W_1}=0.0025$ and $\lambda_{H_1}=0.001$ (no bias is used). All $\W_2$, $\W_1$ and $\H_1$ are initialized with standard normal distribution scaled by 0.1 and are optimized with plain gradient descent with step-size 0.1.
The metrics are computed for $\W=\W_2$ and $\H=\W_1\H_1$. 
We also compute $NC_1$ and $NC_2^{OF}$ for the first layer's features $\H=\H_1$.
The collapse of $\W_1\H_1$ and $\H_1$ to OF (demonstrated by NC1 and NC2 converging to zero) is in agreement with Theorems~\ref{thm_nc_of_deeper}. 

Figure~\ref{fig:app_thm4} %
corroborates Theorem~\ref{thm_nc_of_deeper_nonlin}
that considers the nonlinear model in \eqref{Eq_prob_nonlin}.
We use $K=5, d=20, n=100$,  $\lambda_{W_2}=0.005$, $\lambda_{W_1}=0.0025$, and $\lambda_{H_1}=0.001$ (no bias is used). All $\W_2$, $\W_1$ and $\H_1$ are initialized with standard normal distribution scaled by 0.1, 0.1 and 0.2, respectively, and are optimized with plain gradient descent with step-size 0.1.
The metrics are computed for $\W=\W_2$ and $\H=\sigma(\W_1\H_1)$.  %
We also compute $NC_1$ and $NC_2^{OF}$ for the first layer's features $\H=\H_1$ (as well as for the pre-ReLU $\H=\W_1\H_1$).

Finally, in Figure~\ref{fig:cifar_mse_noBias} we show the similarity of the NC metrics that are obtained for the (nonlinear) extended UFM and metrics obtained by a practical well-trained DNN, namely ResNet18 \citep{he2016deep} (composed of 4 ResBlocks), trained on CIFAR10 dataset via SGD with learning rate 0.05 (divided by 10 every 40 epochs) and weight decay ($L_2$ regularization) of 5e-4, MSE loss and no bias in the FC layer.

\begin{figure}[t]
  \centering
  \begin{subfigure}[b]{0.25\linewidth}
    \centering\includegraphics[width=100pt]{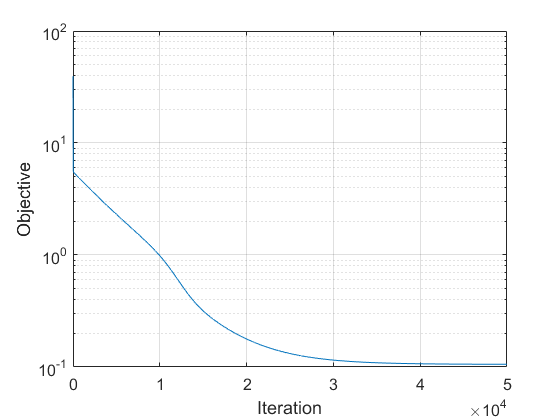}
   \end{subfigure}%
  \begin{subfigure}[b]{0.25\linewidth}
    \centering\includegraphics[width=100pt]{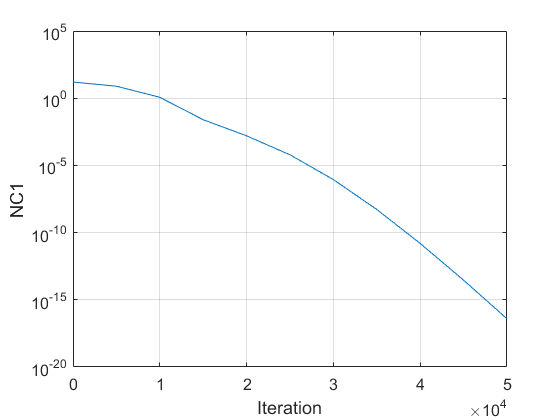}
   \end{subfigure}%
  \begin{subfigure}[b]{0.25\linewidth}
    \centering\includegraphics[width=100pt]{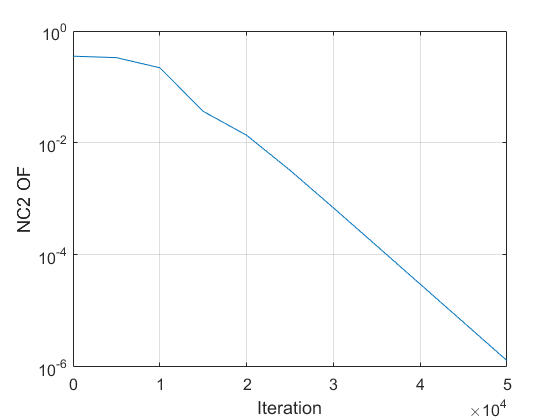}
   \end{subfigure}%
  \begin{subfigure}[b]{0.25\linewidth}
    \centering\includegraphics[width=100pt]{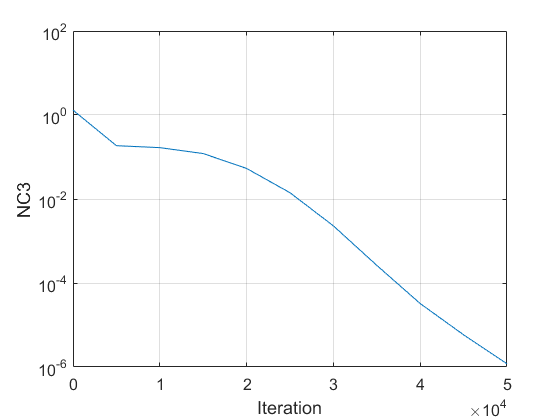}
   \end{subfigure}%
    \caption{
    Verification of Theorem~\ref{thm_nc_of} (MSE loss with no bias).
    From left to right: the objective value, NC1 (within-class variability), NC2 (similarity of the features to OF), and NC3 (alignment between the weights and the features).
    }
\label{fig:app_thm1}     

\vspace{4mm}

  \centering
  \begin{subfigure}[b]{0.25\linewidth}
    \centering\includegraphics[width=100pt]{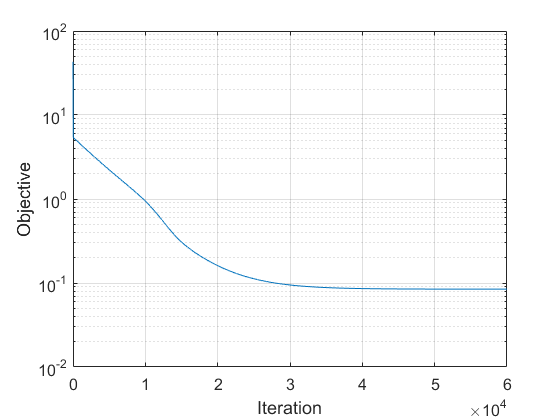}
   \end{subfigure}%
  \begin{subfigure}[b]{0.25\linewidth}
    \centering\includegraphics[width=100pt]{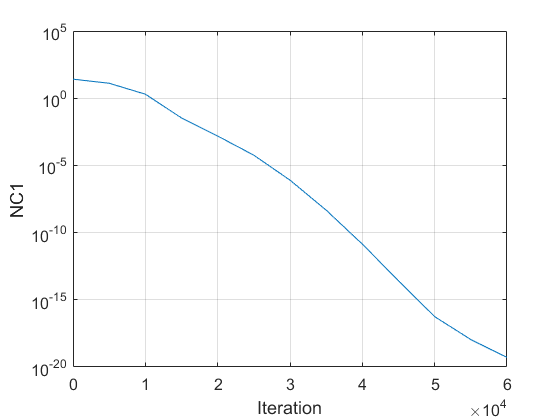}
   \end{subfigure}%
  \begin{subfigure}[b]{0.25\linewidth}
    \centering\includegraphics[width=100pt]{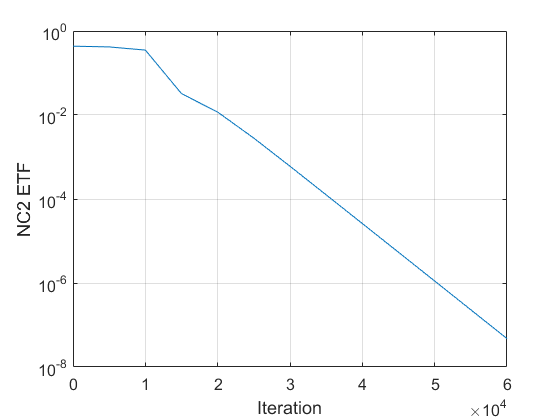}
   \end{subfigure}%
  \begin{subfigure}[b]{0.25\linewidth}
    \centering\includegraphics[width=100pt]{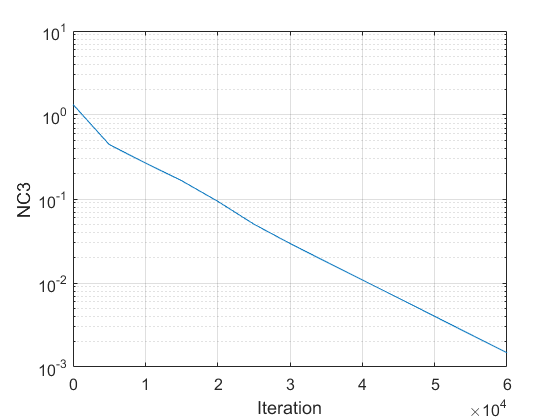}
   \end{subfigure}%
    \caption{
    Verification of Theorem~\ref{thm_nc_simplex_etf} (MSE loss with unregularized bias).
    From left to right: the objective value, NC1 (within-class variability), NC2 (similarity of the features to simplex ETF), and NC3 (alignment between the weights and the features).    
    }
\label{fig:app_thm2}     

\vspace{4mm}

  \centering
  \begin{subfigure}[b]{0.25\linewidth}
    \centering\includegraphics[width=100pt]{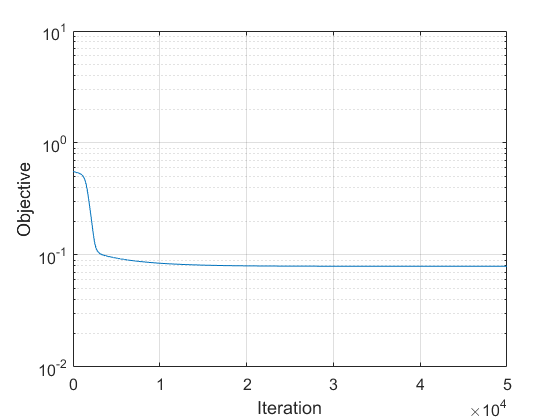}%
   \end{subfigure}%
  \begin{subfigure}[b]{0.25\linewidth}
    \centering\includegraphics[width=100pt]{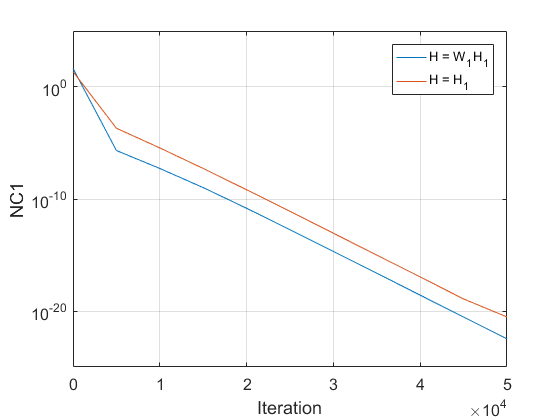}
   \end{subfigure}%
  \begin{subfigure}[b]{0.25\linewidth}
    \centering\includegraphics[width=100pt]{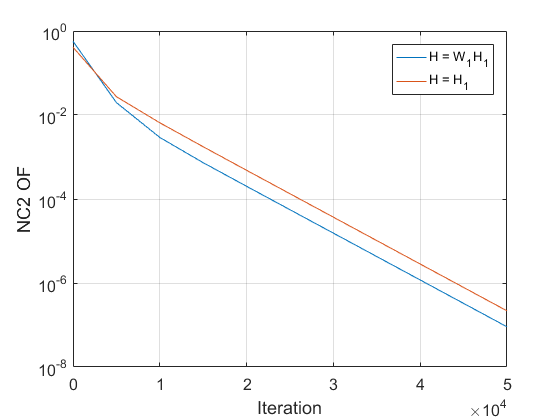}
   \end{subfigure}%
  \begin{subfigure}[b]{0.25\linewidth}
    \centering\includegraphics[width=100pt]{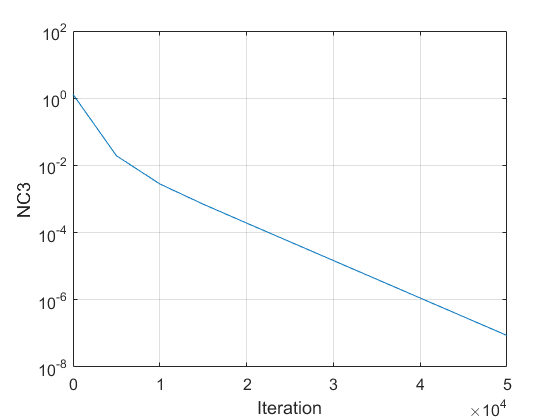}
   \end{subfigure}%
    \caption{
    Verification of Theorem~\ref{thm_nc_of_deeper} %
    (two levels of features).
    From left to right: the objective value, NC1 (within-class variability), NC2 (similarity of the features to OF), and NC3 (alignment between the weights and the features).
    }
\label{fig:app_thm3}     

\vspace{4mm}

  \centering
  \begin{subfigure}[b]{0.25\linewidth}
    \centering\includegraphics[width=100pt]{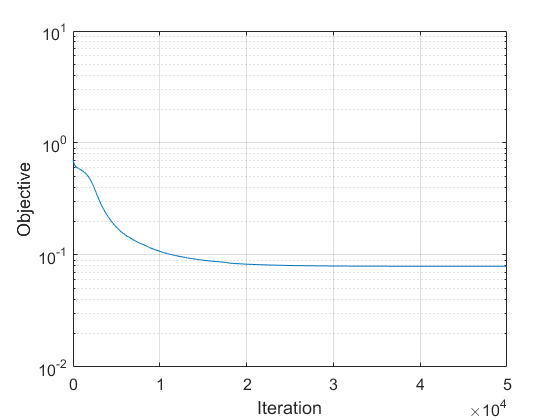}%
   \end{subfigure}%
  \begin{subfigure}[b]{0.25\linewidth}
    \centering\includegraphics[width=100pt]{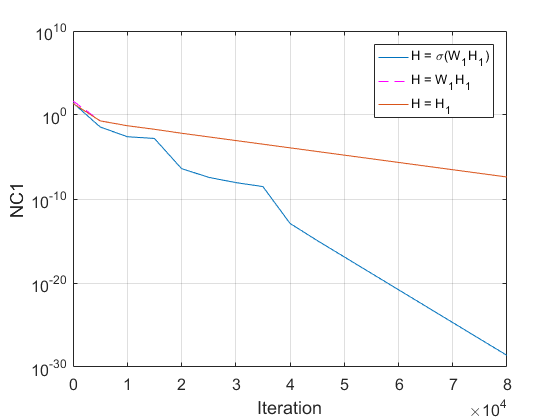}
   \end{subfigure}%
  \begin{subfigure}[b]{0.25\linewidth}
    \centering\includegraphics[width=100pt]{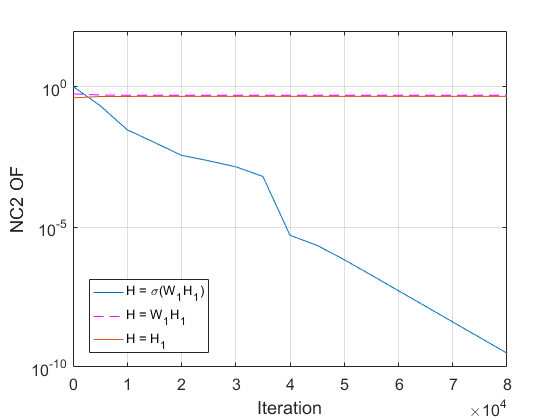}
   \end{subfigure}%
  \begin{subfigure}[b]{0.25\linewidth}
    \centering\includegraphics[width=100pt]{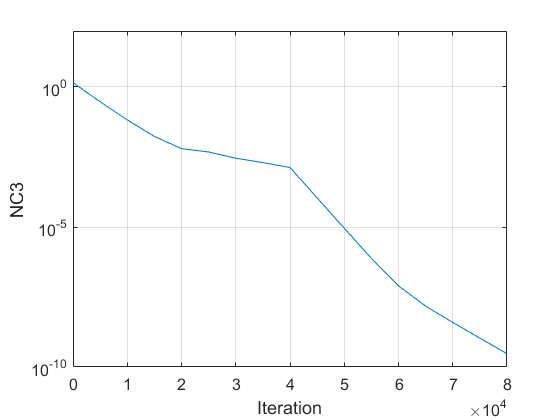}
   \end{subfigure}%
    \caption{
    Verification of Theorem~\ref{thm_nc_of_deeper_nonlin} (two levels of features with ReLU activation).
    From left to right: the objective value, NC1 (within-class variability), NC2 (similarity of the features to OF), and NC3 (alignment between the weights and the features).
    }
\label{fig:app_thm4}     
\end{figure}

\begin{figure}[t]
  \centering
  \begin{subfigure}[b]{0.25\linewidth}
    \centering\includegraphics[width=107pt]{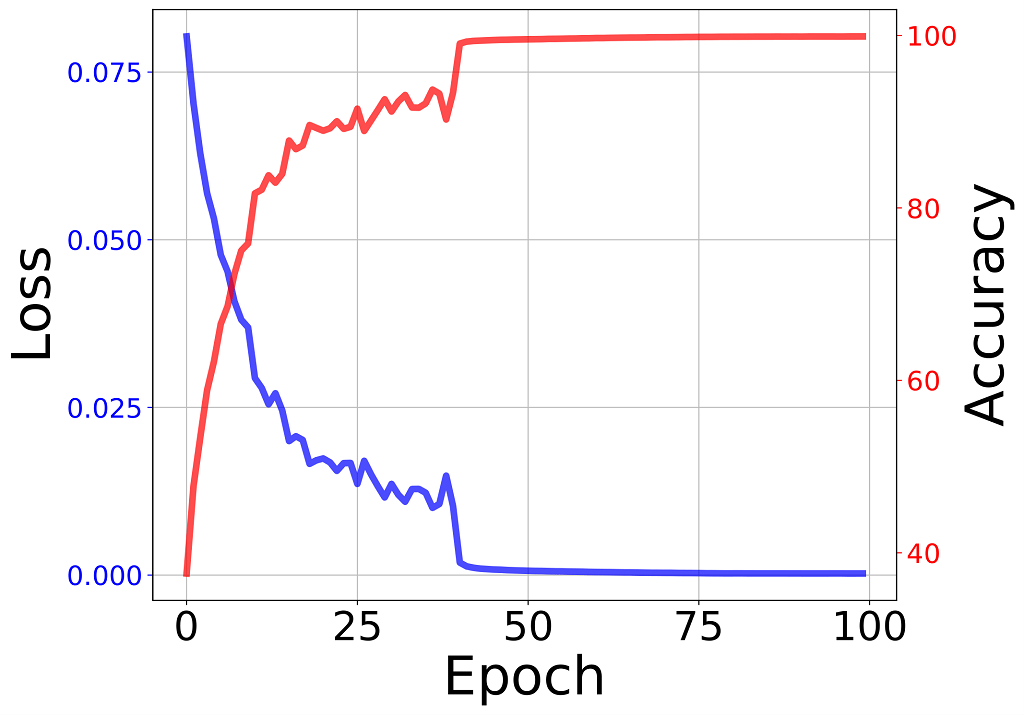}
  \end{subfigure}%
  \begin{subfigure}[b]{0.25\linewidth}
    \centering\includegraphics[width=96pt]{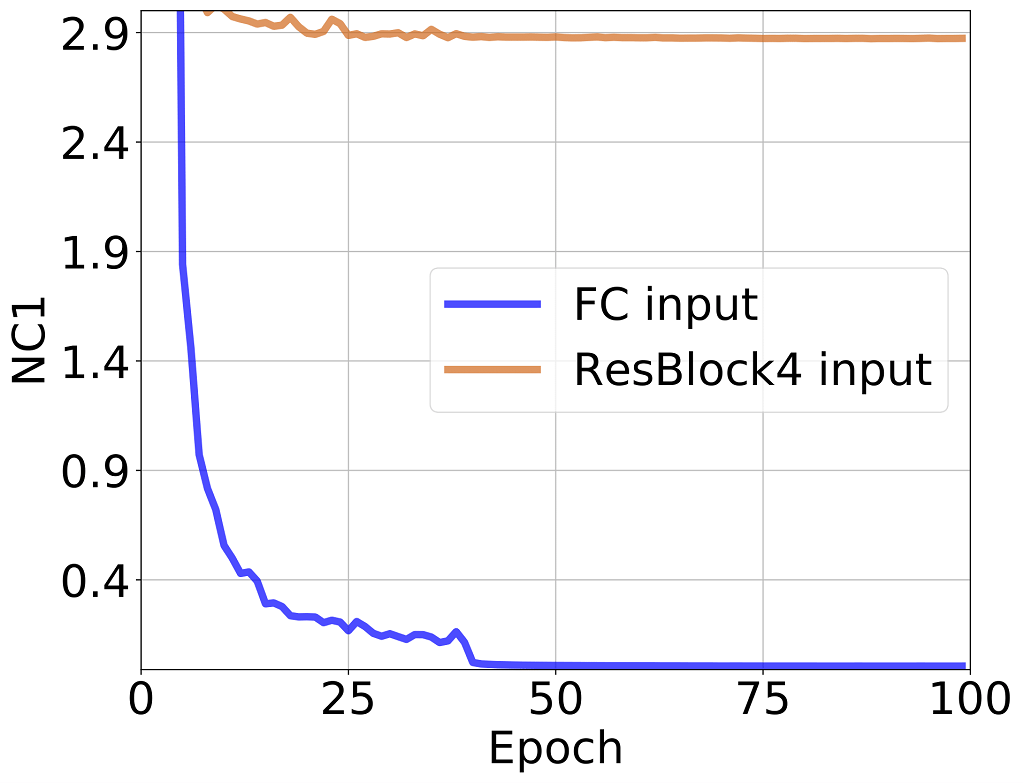}
  \end{subfigure}%
  \begin{subfigure}[b]{0.25\linewidth}
    \centering\includegraphics[width=96pt]{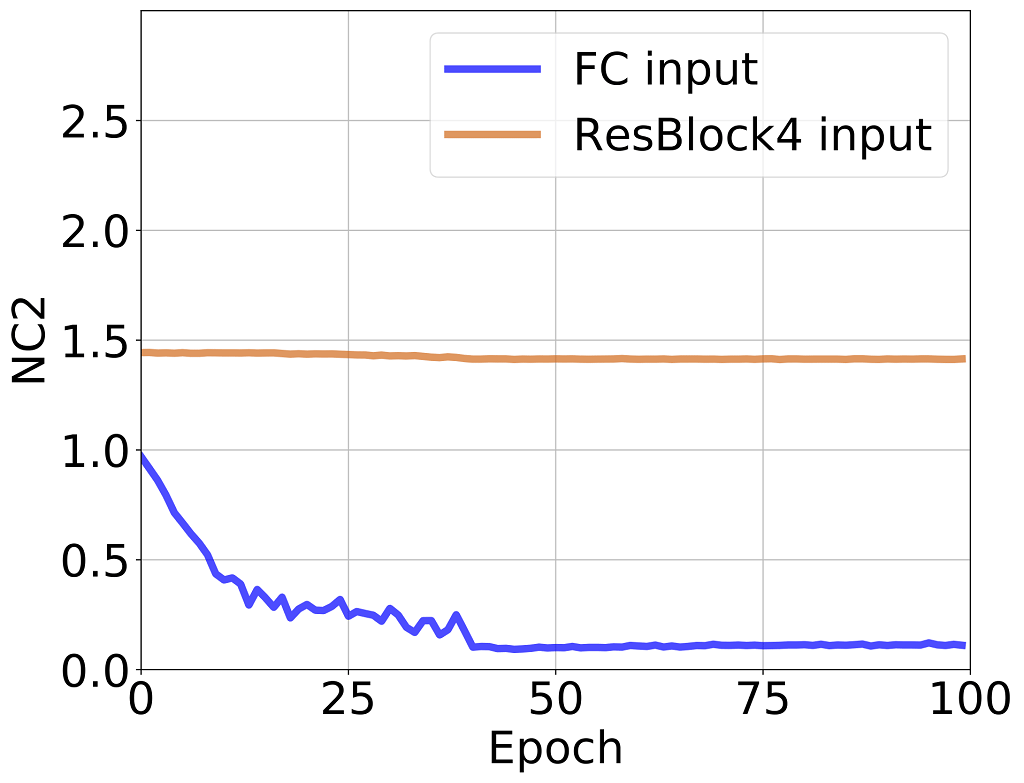}
  \end{subfigure}%
  \begin{subfigure}[b]{0.25\linewidth}
    \centering\includegraphics[width=96pt]{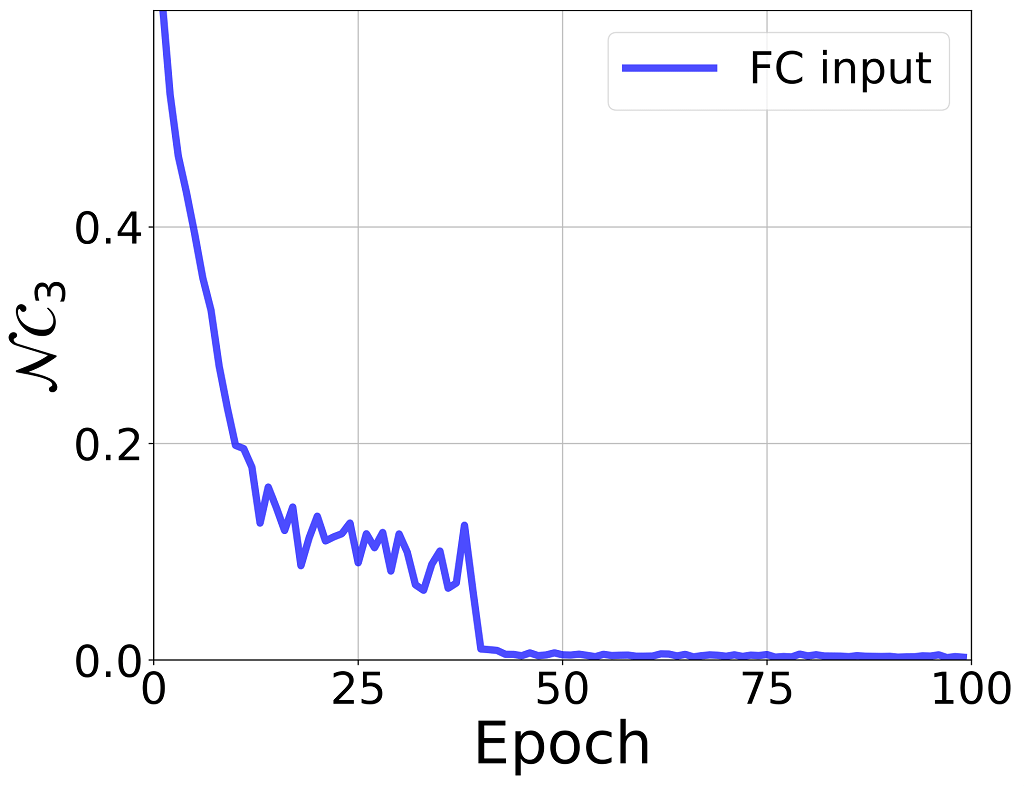}
  \end{subfigure}%
    \caption{
    NC metrics for ResNet18 trained on CIFAR10 with MSE loss, weight decay, and no bias.
    From left to right: training's objective value and accuracy, NC1 (within-class variability), NC2 (similarity of the centered features to simplex ETF), and NC3 (alignment between the weights and the features).    
    }
\label{fig:cifar_mse_noBias}     
\end{figure}

\end{document}